\documentclass{article}

\usepackage{mlmodern}
\usepackage[T1]{fontenc}

\usepackage{url}
\usepackage{mathtools}
\usepackage{amsmath, amssymb, amsfonts, bm, amsthm, bbm,xfrac}
\usepackage{thmtools,thm-restate}
\usepackage{xcolor}
\usepackage{enumitem}
\usepackage{authblk}
\usepackage[margin=1in]{geometry}
\usepackage{natbib}
\usepackage{placeins}
\setcitestyle{authoryear,round,citesep={;},aysep={,},yysep={;}}
\usepackage{graphicx}
\usepackage{subcaption}

\usepackage{hyperref}
\usepackage{cleveref}

\hypersetup{
  colorlinks=true,
  linkcolor=blue,
  citecolor=blue,
  urlcolor=blue
}

\numberwithin{equation}{section}
\theoremstyle{definition}
\newtheorem{theorem}{Theorem}
\newtheorem{proposition}{Proposition}
\newtheorem{lemma}{Lemma}[section]
\newtheorem{corollary}{Corollary}

\newtheorem{remark}{Remark}

\newcommand{\RR}[0]{\mathbb{R}}

\newcommand{\EE}[0]{\mathbb{E}}

\DeclareMathOperator{\diag}{diag}

\DeclareMathOperator*{\argmin}{arg\,min} 

\newcommand{\C}[0]{\mathbb{C}}

\crefname{theorem}{Theorem}{Theorems}
\crefname{lemma}{Lemma}{Lemmas}
\crefname{proposition}{Proposition}{Propositions}
\crefname{corollary}{Corollary}{Corollaries}
\crefname{definition}{Definition}{Definitions}
\crefname{remark}{Remark}{Remarks}
\crefname{figure}{Figure}{Figures}

\author[1]{Arie Wortsman}
\author[2]{Federica Gerace}
\author[1]{Bruno Loureiro}
\author[3]{Yue M. Lu}

\affil[1]{\small Departement d'Informatique, \'Ecole Normale Sup\'erieure, PSL \& CNRS}
\affil[2]{\small Department of Mathematics, University of Bologna,
Piazza di Porta San Donato 5, 40126, Bologna (BO), Italy}
\affil[3]{\small Applied Mathematics, Harvard John A. Paulson School of Engineering and Applied Sciences}

\title{A Random Matrix Theory of Masked Self-Supervised Regression}

\date{}

\begin{document}

\maketitle

\begin{abstract}
In the era of transformer models, masked self-supervised learning (SSL) has become a foundational training paradigm. A defining feature of masked SSL is that training aggregates predictions across many masking patterns, giving rise to a joint, matrix-valued predictor rather than a single vector-valued estimator. This object encodes how coordinates condition on one another and poses new analytical challenges. We develop a precise high-dimensional analysis of masked modeling objectives in the proportional regime where the number of samples scales with the ambient dimension. Our results provide explicit expressions for the generalization error and characterize the spectral structure of the learned predictor, revealing how masked modeling extracts structure from data. For spiked covariance models, we show that the joint predictor undergoes a Baik–Ben Arous–P\'{e}ch\'{e} (BBP)-type phase transition, identifying when masked SSL begins to recover latent signals. Finally, we identify structured regimes in which masked self-supervised learning provably outperforms PCA, highlighting potential advantages of SSL objectives over classical unsupervised methods.
\end{abstract}

\section{Introduction}
Self-supervised learning (SSL) --- a training paradigm in which models learn useful representations from unlabeled data by exploiting the data itself as a source of supervision --- has emerged as a foundational component of the recent success of transformer architectures. By avoiding the need for manual annotations, SSL retains many of the benefits traditionally associated with supervised learning while avoiding reliance on labeled data. Consequently, SSL is widely adopted as a pretraining paradigm for learning general-purpose representations that substantially accelerate the optimization of downstream tasks, especially in data-scarce settings.

A canonical example of a self-supervised learning task is masked language modeling (MLM), in which a neural network is trained to predict masked tokens in text using the remaining tokens as contextual information \citep{devlin2018bert, howard2018universal, radford2018improving, brown2020language, chatgpt}. For example, given the sentence ``\textit{The capital of France is Paris}'', a typical MLM task would be to teach the model to infer that we are speaking about the capital of a country from the context ``France'' and ``Paris'' from the masked sentence ``\textit{The [MASK] of France is Paris}''. Masked language modeling is the core training framework of LLMs in the BERT family \citep{devlin2019bert, liu2019roberta}, where it is used to train state-of-the-art NLP encoders for which bi-directional context is crucial, as well as in vision transformers \citep{bao2021beit, he2022masked}.   

Recent works \citep{rende2024mapping,rende2024distributional} have numerically shown that, when trained with MLM, transformers learn the conditional distribution of a missing token given the observed ones by sequentially inferring positional and semantic relationships across the entire text. However, several fundamental questions remain open. For instance: \emph{How much data is required to achieve good generalization performance in MLM? How does this performance depend on the underlying structure of the data, and in particular on the correlations among tokens?}

In this work, we investigate masked self-supervised learning in its simplest declination: learning real-valued sequence data with a matrix-valued linear predictor. More precisely, given $n$ real-valued sequences $x_1, \dots, x_n \in \RR^{d}$, where $d$ corresponds to the length of the sequence, we train a family of ridge-regularized linear predictors $X \mapsto X\hat{A} $ under the constraint that no coordinate of the sequence can be used to predict itself (this will be made precise in the next section). In this context, \textbf{our main contributions} are: 
\begin{itemize}
\item \textbf{Asymptotic limit for training and generalization errors:} We derive a sharp asymptotic characterization of the training and generalization performance of the aggregate predictor $\hat{A}$ in the high-dimensional regime where $n,d \to \infty$ with a fixed ratio $\sfrac{n}{d} \to \alpha > 0$. The resulting asymptotic limit depends only on the population covariance of the data, yielding a precise and interpretable description of which structural properties of the data enable successful generalization in masked self-supervised regression (SSR).

\item \textbf{Deterministic equivalent for the predictor:}
Beyond performance metrics, we establish a high-dimensional deterministic equivalent for (the resolvent of) the matrix-valued aggregate predictor $\hat{A}$ itself, which allows us to characterize how self-supervised learning encodes and exploits the underlying data geometry. Technically, this result relies on the analysis of a random matrix defined as an aggregate of a diverging number of correlated predictors. This matrix does not belong to standard random-matrix ensembles, and our analysis may be of independent interest.

\item \textbf{Spiked vs AR(1):}
We further provide an in-depth analysis of two prototypical statistical models. In a \textbf{spiked covariance model}, we show that principal component analysis (PCA) strictly outperforms masked self-supervised learning, and we establish a BBP-type transition in the asymptotic spectrum of $\hat{A}$ occurring at the same threshold as the corresponding transition for the sample covariance matrix. In sharp contrast, for a \textbf{first-order autoregressive model}, we prove that masked self-supervised regression can strictly dominate PCA in performance if the number of directions in PCA is not close to the dimension. We further provide a detailed analysis of the effective inductive bias induced by strong temporal correlations, elucidating the mechanisms by which self-supervision leverages sequential structure.
\end{itemize}
Together, our results elucidate the mechanisms by which masked self-supervised learning exploits correlations in the training data, and provide a principled comparison between SSR and classical spectral approaches, revealing their respective strengths and limitations in the high-dimensional regime.

\paragraph{Further Related Works: } Deterministic equivalents for sample covariance matrices have a long history in the random matrix theory literature, starting with the seminal work of \citep{marvcenko1967distribution} and followed by several generalizations to the anisotropic settings \citep{bai2008large, burda2004signal,knowles2017anisotropic,rubio2011spectral,louart2018concentration,chouard2022quantitative}.

In the context of statistics and machine learning theory, RMT techniques have been used to derive insight into a manifold of high-dimensional models, including linear regression \citep{DW18, wu2020optimal, HMRT22, CM24}, kernels \citep{hu2024asymptotics,xiao2022precise, misiakiewicz2024non, lu2025equivalence}, random features \citep{pennington2017nonlinear,louart2018random,liao2018spectrum, mei2022generalization,d2020double,fan2020spectra,benigni2021eigenvalue,schroder2023deterministic,schroder24asymptotics, DLM24,atanasov2024scaling, vigeant2026dyson}, one-step analysis of feature learning \citep{ba2022high,moniri2023theory,dandi25random} and in-context learning \citep{lu2024context}. In the unsupervised context, RMT has been used to study mainly PCA \citep{BBAP05, RW20} \citep{RW20} and clustering methods \citep{CL22}. More recently, RMT techniques were used to investigate sequential and time series data in \citep{ilbert2024analysing}, and a matrix-valued least-squares problem in \citep{fan2024kronecker}.

\paragraph{Notation:} Given a matrix $X \in \RR^{n \times d}$,  we consistently index by $i,j \in [n]$ the rows of $X$, and by $k, l\in[d]$ its columns, e.g. $X=(x_{ik})_{\substack{1\leq i\leq n\\ 1\leq k\leq d}}$. We denote by $X_{i}\in\mathbb{R}^{d}$ the $i$th row of $X$ and $X_{k}\in\mathbb{R}^{n}$ the $k$th column of $X$. In the same way, for $i\in[n]$, we denote by $X^{(-i)}\in\mathbb{R}^{(n-1)\times d}$ the matrix obtained by removing the $i$th row, and by $X^{(-k)}\in\mathbb{R}^{n\times (d-1)}$ the matrix obtained by removing the $k$th column. 
\section{Setting}
\label{sec:setting}
Let $x_1, \dots, x_n \in \RR^{d}$ denote $n$ data points in $\RR^{d}$, for $n,d \in \mathbb{N}$, and let $X =[x_1, \dots, x_n]^{\top} \in \RR^{n \times d}$ denote the data matrix. Here, each $x_{i}$ can be thought as a sequence of (real valued) tokens or pixels of length $d$, in the case of text or image data, for instance. 

As motivated in the introduction, we are interested in the problem of masked self-supervised learning: given a sequence $x \in \RR^{d}$ and an arbitrary coordinate $k \in [d]$ (the ``\emph{mask}''), the goal is to predict the masked coordinate $x_k$ from the remaining  $[d]\setminus \{k\}
$ coordinates. Here, we focus on the simplest instance of this problem where we perform self-supervised ridge regression with regularization $\lambda>0$:
\begin{align}
    \hat{a}_{k} & =  \underset{\substack{a\in\mathbb{R}^{d}\\ a_{k}=0}}{\argmin}\frac{1}{n}||X_{k}-Xa||_{2}^{2}+\lambda||a||_{2}^{2},
    \label{eq:coordinate_wise_regression}
\end{align}
where the constraint $a_k = 0$ makes sure that we do not use the $k$th coordinate to predict itself.
\begin{remark}
    The problem in \cref{eq:coordinate_wise_regression} corresponds exactly to the regularized maximum likelihood problem for the case in which the sequence $x=(x_{1},\dots,x_{d})$ is jointly Gaussian and the variance is fixed. 
\end{remark}
\Cref{eq:coordinate_wise_regression} is a strictly convex problem for $\lambda>0$ --- denote by $\hat{a}_{k}$ its minimizer for each $k \in [d]$. We can define the following aggregate prediction matrix: 
\begin{equation}
    \hat{A}: = [\hat{a}_1, \hat{a}_2, \dots, \hat{a}_d] \in \RR^{d \times d}. 
    \label{eq:A_d_different_estimators}
\end{equation}
Note that by construction $\mathrm{diag}(\hat{A})=0$. This matrix, which we will refer to as the \emph{self-supervised ridge matrix} (SSR matrix), contains aggregated information about the correlations in the training data $X$. 
\begin{remark}[Relationship with factored attention]
    As discussed in \citep{rende2024mapping}, for this real-valued sequence task the matrix-valued predictor $\hat A\in\mathbb{R}^{d\times d}$ can be interpreted as the attention map of a single layer of \emph{factored self-attention}, with an untrained value matrix fixed to the identity. Despite its simplicity, factored attention has been empirically shown to outperform standard attention mechanisms in some contexts, including protein contact prediction \citep{bhattacharya2020single} and the approximation of ground states of many-body quantum systems \citep{viteritti2022transformer,rende2023simple,viteritti2023transformer}. 
\end{remark}
In the following, we will be mainly interested in two questions: (i) what properties of the data are captured by the SSR matrix $\hat{A}$? (ii) How well does $\hat{A}$ perform reconstruction of the masked entries at a fixed training data budget, as quantified by the total reconstruction population risk:
\begin{equation}
    L(\hat{A}) = \frac{1}{d}\mathbb{E}_x\left[ \|x - \hat{A}^{\top}x \|^2 \right ],
    \label{eq:gen_error_expression}
\end{equation}
where the expectation is taken over a new sequence $x \in \mathbb{R}^d$ that is not in the training set. 
\begin{remark}
    The total reconstruction error is also commonly used as the generalization error in other unsupervised learning tasks, in particular for PCA [see, e.g \citep{RW20} and  \citep{EHEM25}].
\end{remark} 
Complementary, we also define the total reconstruction empirical risk:
\begin{equation}
    \hat{L}_{n}(\hat{A}) = \frac{1}{nd}||X-XA||^{2}_{F}, 
    \label{eq:train_error}
\end{equation}
where $X \in \mathbb{R}^{n \times d}$ is the training data matrix in \Cref{eq:coordinate_wise_regression}. We will refer to \Cref{eq:gen_error_expression} and \Cref{eq:train_error} as the generalization and training errors, respectively. 

Our main results provide sharp answers to both questions (i) and (ii) in the \emph{high-dimensional limit}, where the number of samples $n$ and the sequence length $d$ grow to infinity proportionally, with $\sfrac{n}{d}\to\alpha = \Theta_{d}(1)$ denoting the sample complexity.  More precisely, we will derive a \emph{deterministic equivalent} for the resolvent of the SSR matrix $\hat{A}$ and asymptotic limits of the corresponding risks $L(\hat{A})$, $\hat{L}_{n}(\hat{A})$. These results allow us to precisely characterize how the aggregated predictor encodes the geometry of the underlying data distribution, and how this geometry governs the total reconstruction error.

\section{Main Results}
Our first result concerns the structural properties of the aggregated predictor matrix $\hat A$ defined in \cref{eq:A_d_different_estimators}. While vector-valued ridge estimators of the form \cref{eq:coordinate_wise_regression} have been extensively analyzed in the random matrix theory literature, the aggregate estimator studied here exhibits fundamentally different behavior. In particular, although each coordinate-wise minimizer $\hat a_k$ solves a standard ridge regression problem, all such estimators are trained on the same data matrix $X$, inducing strong statistical dependencies across coordinates. As a consequence, existing RMT results that treat each $\hat a_k$ in isolation are insufficient to characterize the joint behavior of the matrix-valued estimator $\hat A$, its spectrum, and the associated reconstruction risk. A global analysis that explicitly accounts for these correlations is therefore required. To that end, we first derive a joint expression for the minimizer $\hat{A}$: 

\begin{lemma}\label[lemma]{lemma:explicit_expression_A}
    For $\lambda >0$, let $\hat{\Sigma} = \frac{1}{n} X^{\top} X \in \RR^{d \times d}$ denote the sample-covariance matrix and let 
    \[
    Q(\lambda) = (\hat{\Sigma} + \lambda I_{d})^{-1}
    \]
    denote the resolvent of $\hat{\Sigma}$.  Then
    \[
    \hat{A} = I_{d} - Q(\lambda) \Lambda, \qquad \Lambda =[\mathrm{diag} (Q(\lambda)) ]^{-1}.
    \]
\end{lemma}
We refer the reader to Appendix~\ref{appendix_preliminaries} for the proof. 
\begin{remark}
    Note that the matrix $\hat{A}$ is, in general, not symmetric as $Q(\lambda)$ and $\Lambda$ do not commute in general. However,  by applying the similarity transformation $S \to \Lambda^{\frac{1}{2}} S \Lambda^{-\frac{1}{2}}$ to the matrix $Q(\lambda) \Lambda$, we obtain the matrix $\Lambda^{\frac{1}{2}}Q(\lambda) \Lambda^{\frac{1}{2}}$, which is a symmetric matrix. Then $\hat{A}$ has the same spectrum as $I_d -  \Lambda^{\frac{1}{2}}Q(\lambda) \Lambda^{\frac{1}{2}}$, and therefore has a real-valued spectrum. 
    \label{rmk:real_eigenvalues_A}
\end{remark}
\Cref{lemma:explicit_expression_A} shows that $\hat A$ depends on the data only through the resolvent of the empirical covariance matrix. This representation naturally suggests an approach based on random matrix theory. The main technical challenge, however, is that this dependence appears through the product $Q(\lambda)\Lambda$, where $\Lambda$ is a diagonal matrix defined as a nonlinear function of the resolvent. Although $\Lambda$ admits a deterministic equivalent in the high-dimensional limit, its finite-dimensional dependence on $Q(\lambda)$ precludes a direct application of standard random matrix results. Our analysis therefore proceeds by first establishing a sharp concentration of $\Lambda$ around its deterministic limit, and then leveraging a linearization argument to control the resulting structured product and characterize the behavior of the aggregated predictor.

\subsection{The Generalization and Training Errors}
With the description of $\hat{A}$ in \Cref{lemma:explicit_expression_A}, we can characterize the generalization properties of the SSR predictor. A first natural statistical question is whether this is a consistent estimator, i.e. whether the estimator $\hat{A}$ reaches vanishing generalization error in the classical statistical limit where $n\to\infty$ at fixed $d=\Theta_{n}(1)$. The answer is negative: To see this, recall that given $k \in [d]$, we can always decompose 
$x_{k} = \mathbb{E}[x_{k} | x^{(-k)}] + (x_k - \mathbb{E}[x_{k} | x^{(-k)}])$, and $\mathbb{E}[x_{k} | x^{(-k)}]$ is the best estimator of $x_k$ given the rest of the coordinates. Then, if $x_{k} - \mathbb{E}[x_{k} | x^{(-k)}]$ is non-zero, there will always be a part that is not predictable by the rest of the coordinates. As an example, if the data were sampled from a Gaussian distribution with a full-rank covariance, then $x_{k} - \mathbb{E}[x_{k} | x^{(-k)}]$ is always non-zero. 
\begin{lemma}[Approximation error]
    \label[lemma]{lemma:approximation_error}
    Assume $\Sigma:=\mathbb{E}[xx^{\top}] \succ 0$. Then, the approximation error is given by:
    \begin{align}
        L_{\rm App} := \underset{\substack{A\in\mathbb{R}^{d\times d}\\ {\rm diag}(A)=0}}{\min} L(A) = \frac{1}{d}\mathrm{Tr} (\Sigma^{-1} [\diag(\Sigma^{-1})]^{-2}) = \dfrac{1}{d} \sum_{k \le d} \frac{1}{(\Sigma^{-1})_{kk}}. 
    \end{align}
    and it is attained by $A_\mathrm{App} : = I_{d} -  \Sigma^{-1}[\diag(\Sigma^{-1})]^{-1}$. 
\end{lemma}

\begin{remark}
    The approximation error defined above is the best achievable error in the SSR hypothesis: a linear predictor in the form of a square matrix with zero-diagonal.
\end{remark}

\begin{remark}
    For Gaussian Data, the Approximation error coincides with the sum of the Bayes errors for each coordinate prediction in \Cref{eq:coordinate_wise_regression}. To see this, note that for each $k$, we have that the Bayes error of the predictor $\hat{a}_k$ in \Cref{eq:coordinate_wise_regression} is achieved by the estimator $\mathbb{E}\left[x_{k} | x_{-k}\right ]$. When $x \sim \mathcal{N}(0,\Sigma)$, for $\Sigma \succ 0$, we have: 
    \begin{equation}
        \mathbb{E} \left [ (x_{k} - \mathbb{E}\left[x_{k} | x_{-k}\right ])^{2} \right ] = \mathrm{Var}(x_{k}| x_{k-1}) = \dfrac{1}{(\Sigma^{-1})_{k,k}},
    \end{equation}
    and the sum of the Bayes errors for each coordinate becomes
    the expression in \Cref{lemma:approximation_error}. 
\end{remark}

\begin{figure}[t]
    \centering
    \includegraphics[width=0.7\linewidth]{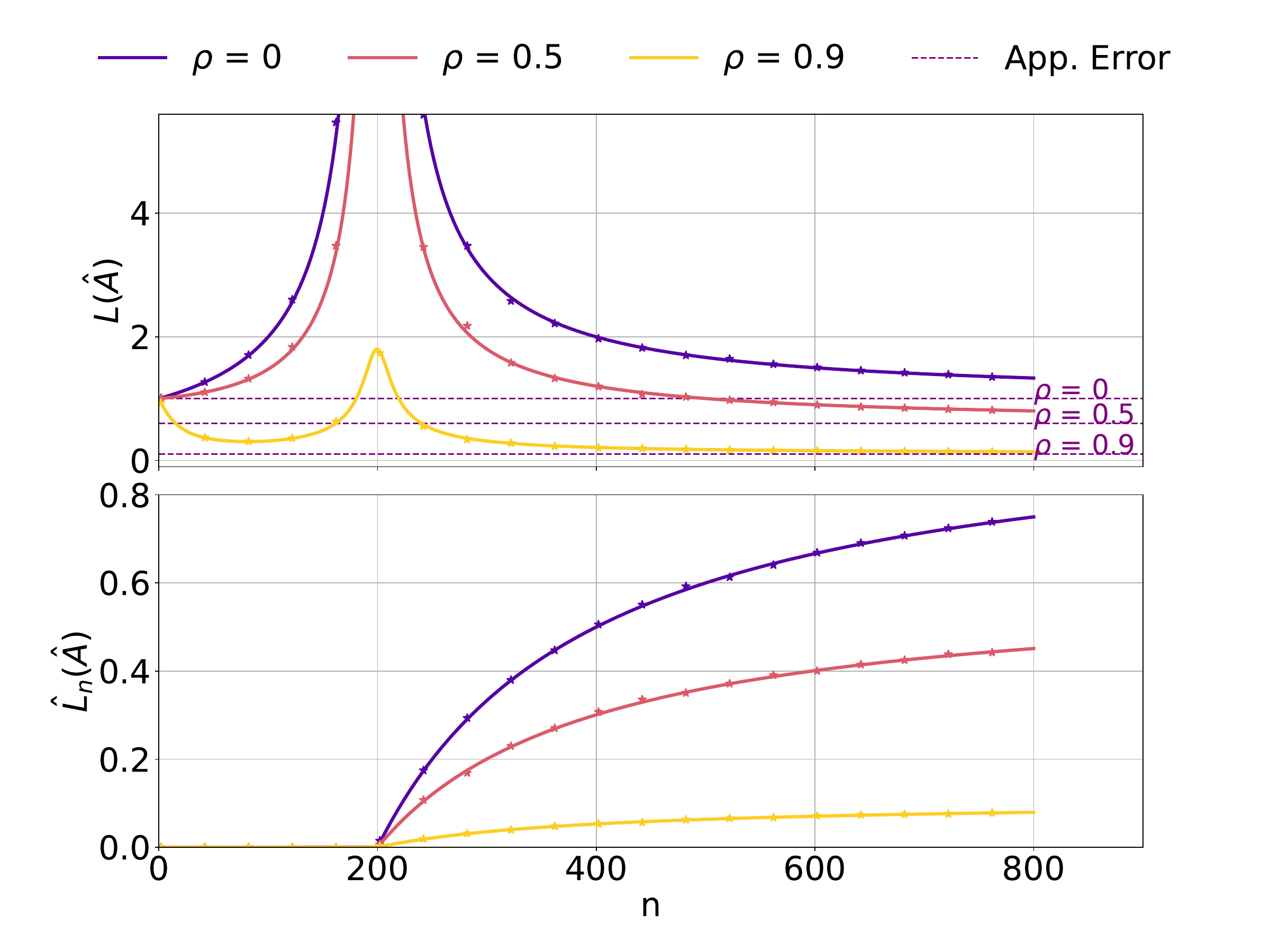}
    \caption{Generalization and Training Error for the SSR estimator for Gaussian data with a Toeplitz covariance $\Sigma_{i,j} = \rho^{|i-j|}$, for different values of $\rho$. The value $\rho=0$ corresponds to $\Sigma=I_d$. Solid lines correspond to the asymptotic limit, while dots correspond to the empirical error. In all curves, $d = 200$ and $\lambda = 10^{-4}$.} 
    \label{fig:deterministic_equivalent_gen__and_train_error}
\end{figure}

\Cref{lemma:approximation_error} shows that in general the lowest achievable error is non-zero. Moreover, note that in the particular case where $\Sigma$ is a diagonal matrix, the best possible estimator is $\hat{A}\equiv 0$. This is intuitive, as when coordinates are not correlated there is no information on a coordinate in the remaining. This highlights that structure is fundamental for successful SSR, and raises the question of what properties of the data geometry are particularly important. To get an intuition for this question, consider the spectral decomposition of the population covariance $\Sigma = \sum_{k=1}^{d}\lambda_{k}u_{k}u_{k}^{\top}$, and note the approximation error can be rewritten as:
\begin{equation}
    L_\mathrm{App} = \frac{1}{d}\sum_{k=1}^{d} \dfrac{1}{\lambda_{k}} u_{k}^{\top} [\mathrm{diag}(\Sigma^{-1})]^{-2} u_{k}. 
\end{equation}
Taking inspiration from the case of Gaussian data, we have that for $k \in [d]$, $\mathrm{Var}(x_k | x_{-k}) = \sfrac{1}{(\Sigma^{-1})_{kk}}$. Hence, the matrix $[\mathrm{diag}(\Sigma^{-1})]^{-1}$ can be interpreted as a measure of how predictable (or explainable) each coordinate is with respect to the rest. With this interpretation, we see that the quadratic form $u_{k}^{\top} [\mathrm{diag}(\Sigma^{-1})]^{-2} u_{k}$ penalizes the magnitude of the vectors $u_{k}$ with this variance. Therefore, this quantity is lower when the coordinates of the eigenvectors are delocalized, and $\mathrm{Var}(x_k | x_{-k})$ is large for each $k \in [d]$. 

Now that we understand what the best achievable error is, we move to addressing the performance of SSR at finite sample complexity. Our first main result is to show that the risk achieved by the SSR predictor $L(\hat{A})$ admits a fully deterministic characterization in the proportional high-dimensional limit where $d\to\infty$ with $n=\Theta(d)$, depending only on the population covariance $\Sigma$, the asymptotic ratio $\sfrac{n}{d}\to\alpha$ and the regularization strength $\lambda>0$.
\begin{theorem}[Asymptotic Limit of the risk]
    Let $X = Z \Sigma^{\frac{1}{2}} \in \RR^{n \times d}$, with $\Sigma \succeq 0$, and $\| \Sigma \|_\mathrm{op}$ bounded, and $Z$ having independent entries, with mean $0$, variance $1$, and $4 + \varepsilon$-th bounded moments. Let $\mathrm{df}_1^{\Sigma} := \mathrm{Tr}(\Sigma(\Sigma+\kappa I_d)^{-1})$ and $\mathrm{df}_2^{\Sigma} := \mathrm{Tr}(\Sigma^2(\Sigma+\kappa I_d)^{-2})$ denote the degrees of freedom.   Let $\kappa_{\star}(\lambda)$ denote the (unique) solution of the self-consistent equation: 
    \begin{equation}
        \kappa = \lambda + \frac{\kappa}{n} \mathrm{df}_{1}^{\Sigma}(\kappa).
        \label{eq:equation_kappa_lambda}
    \end{equation}
     Then, as $n,d\to\infty$ at fixed $\sfrac{n}{d}\to\alpha=\Theta_{d}(1)$, the normalized generalization error converges in probability to:
    \begin{align*}
    L(\hat{A}) & \to  L_1 \left (1 +  \dfrac{\mathrm{df}^{\Sigma}_2(\kappa_{\star}(\lambda)) }{n - \mathrm{df}^{\Sigma}_2(\kappa_{\star}(\lambda))} \right ) 
    \end{align*}
    where 
    \begin{align*}
        L_1  = \frac{1}{d}\mathrm{Tr} \left [ \bar{D}^2 (\Sigma + \kappa_{\star}(\lambda) I_{d})^{-1} \Sigma (\Sigma + \kappa_{\star}(\lambda) I_{d})^{-1}\right ], 
    \end{align*}
    and 
    \begin{align*}
        \bar{D} = \bar{D}(\Sigma, \kappa_{\star}(\lambda) ) :=[\mathrm{diag}((\Sigma + \kappa_{\star}(\lambda) I_d)^{-1})]^{-1}.
    \end{align*} 
    \label[theorem]{theorem:det_equivalent_gen_error}
\end{theorem}
\begin{remark}
    It is important to highlight that computing this asymptotic limit from existing limits in the literature for the minimizer $\hat{a}_k, k \in [d]$ (e.g.,  \citep{DW18,wu2020optimal,HMRT22}) would require summing a diverging number of individual formulas, one for each $k \in [d]$. Instead, \Cref{theorem:det_equivalent_gen_error} gives us a closed formula for the loss associated with the aggregate predictor.   
\end{remark}
To get some intuition for the formulas in \Cref{theorem:det_equivalent_gen_error}, it is instructive to consider the ridge-less limit $\lambda \to 0^{+}$ when $\alpha >1$. In this case, $\kappa(\lambda) \to 0$, so $L_{1} \to L_\mathrm{App}$ defined in \Cref{lemma:approximation_error}, which is independent of $n$. Then:
\begin{equation}
    L(\hat{A}) - L_{1} \sim L_{1} \dfrac{\mathrm{df}_2(0) }{n - \mathrm{df}_2(0)} = L_{\rm App}\dfrac{1}{\sfrac{n}{d} - 1}. 
\end{equation}
From the above, we can see an explicit dependence of the generalization error on $\alpha$. More generally, we can study how $L(\hat{A})$ converges to the approximation error by computing the degrees of freedom, which depend only on the eigenvalues of $\Sigma$. Then, the general performance of the SSR estimator will be determined by $L_\mathrm{App}$, which is a highly structure-dependent quantity (as described above), and the speed at which the convergence will happen depends on the eigenvalues of $\Sigma$. 

Away from this particular limit, we can numerically verify \Cref{theorem:det_equivalent_gen_error} 
against finite-size simulations.  The top \Cref{fig:deterministic_equivalent_gen__and_train_error} shows the generalization error for self-supervised ridge regression for Gaussian data with a structured covariance $\Sigma$, which corresponds to a linear auto-regressive process. We plot both the empirical generalization error and the asymptotic limit of \Cref{theorem:det_equivalent_gen_error}, showing a remarkable agreement even for modest sizes $d=200$. The dotted lines below represent the approximation error given by \Cref{lemma:approximation_error}. Note that the behavior of the generalization error varies depending on the value of $\rho$, changing both behavior of the curves and its asymptotic value. This model will be discussed in depth in \cref{sec:toeplitz}.
Finally, a similar characterization as the one in \Cref{theorem:det_equivalent_gen_error} holds for the training error defined in \Cref{eq:train_error}. 
\begin{theorem}
    Under the same assumptions and notation as \Cref{theorem:det_equivalent_gen_error}, as $n,d$ grow to $\infty$ proportionally as $\alpha = \frac{n}{d}$, the normalized training error converges in probability to 
    \begin{align*}
    \hat{L}(\hat A) & \to  \dfrac{\kappa(\lambda)}{\lambda d}\bar{L}_1 - \dfrac{\kappa(\lambda)^2}{\lambda d}\bar{L}_2 - \dfrac{\kappa(\lambda)^2}{\lambda d(n - \mathrm{df}_2^{\Sigma}(\kappa(\lambda)))}\bar{L}_3,
    \end{align*}
    where, denoting $\bar{Q}(\kappa) = (\Sigma + \kappa I_d)^{-1}$, we have: 
    \begin{align*}
    \bar{L}_1 = \mathrm{Tr} \left (\bar{D}^2 \bar{Q}(\kappa)\right );  \quad 
    \bar{L}_2 =  \mathrm{Tr} \left (\bar{D}^2 \bar{Q}(\kappa)^2 \right ),
    \end{align*}
    and 
    \[
    \bar{L}_3 = \mathrm{Tr} \left (\bar{D}^2 \bar{Q}(\kappa)^2\Sigma \right )\mathrm{Tr} \left (\bar{Q}(\kappa)^2 \Sigma \right ). 
    \]
    \label[theorem]{theorem:det_equivalent_train_error}
\end{theorem}
\Cref{theorem:det_equivalent_train_error} is illustrated at the bottom of \Cref{fig:deterministic_equivalent_gen__and_train_error}. For all values of $\rho$, the asymptotic limit in \Cref{theorem:det_equivalent_train_error} shows excellent agreement with the finite-size training error. From this figure, it is clear that the origin of the double descent peak in \cref{fig:deterministic_equivalent_gen__and_train_error} coincides with the point in which the estimator interpolates the training data, akin to the standard double descent phenomena in supervised learning \citep{mei2022generalization, gerace2020generalisation, HMRT22}. 

\subsection{Asymptotic spectrum of the SSR matrix}
Beyond performance characterization, a natural question is how the structure of the data distribution is reflected in the SSR estimator $\hat A$. As shown in \cite{rende2024mapping}, for simple data models the attention matrix learned by single-layer architectures trained with masked self-supervised objectives is closely related to the inverse of the data covariance. This behavior is closely aligned with the structure that emerges in our self-supervised ridge regression setting.

\begin{figure}[t]
    \centering
    \begin{subfigure}{0.48\textwidth}
        \centering
        \includegraphics[width=\linewidth]{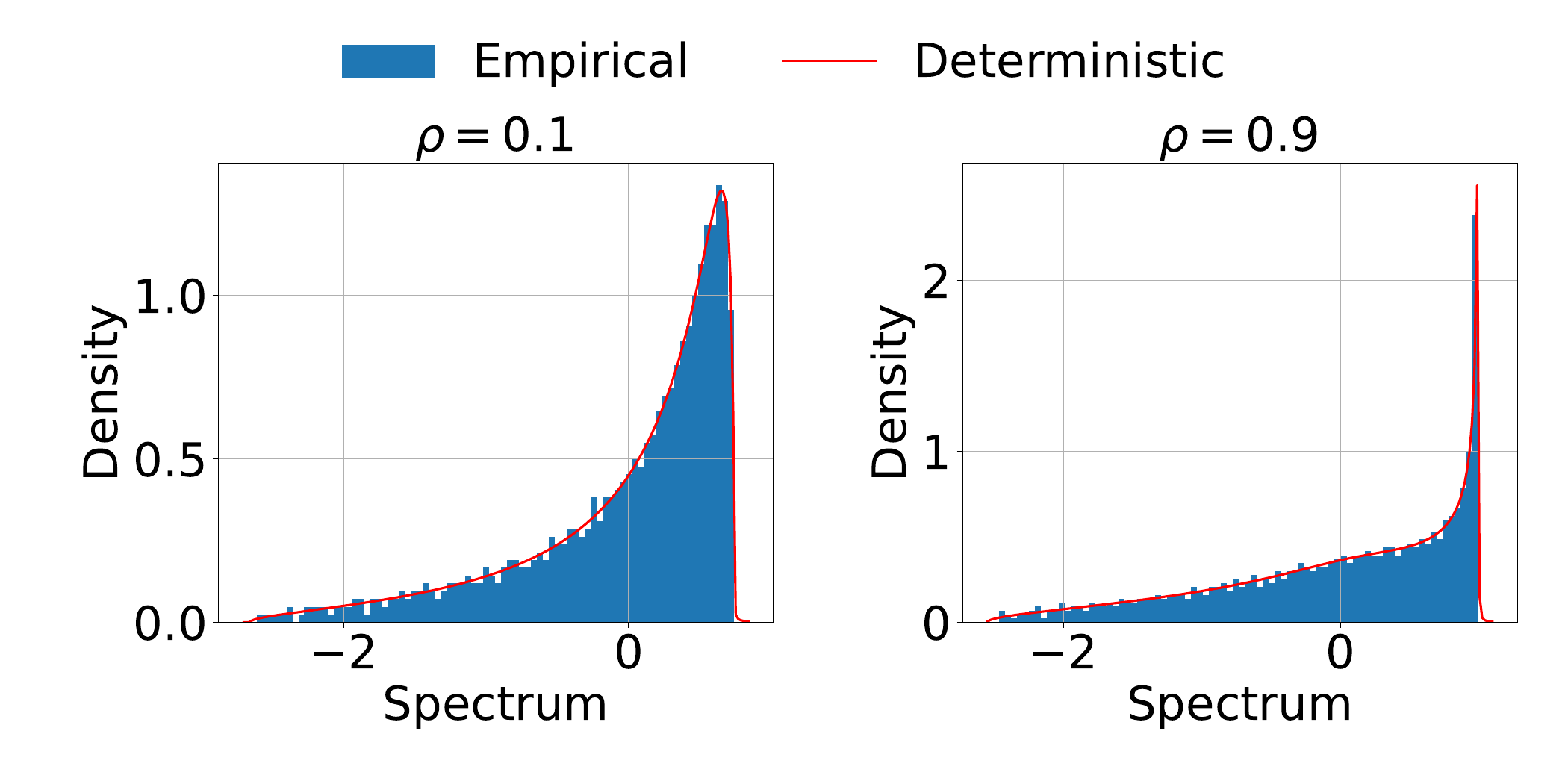}
        \label{fig:first}
    \end{subfigure}
    \hfill
    \begin{subfigure}{0.48\textwidth}
        \centering
        \includegraphics[width=\linewidth]{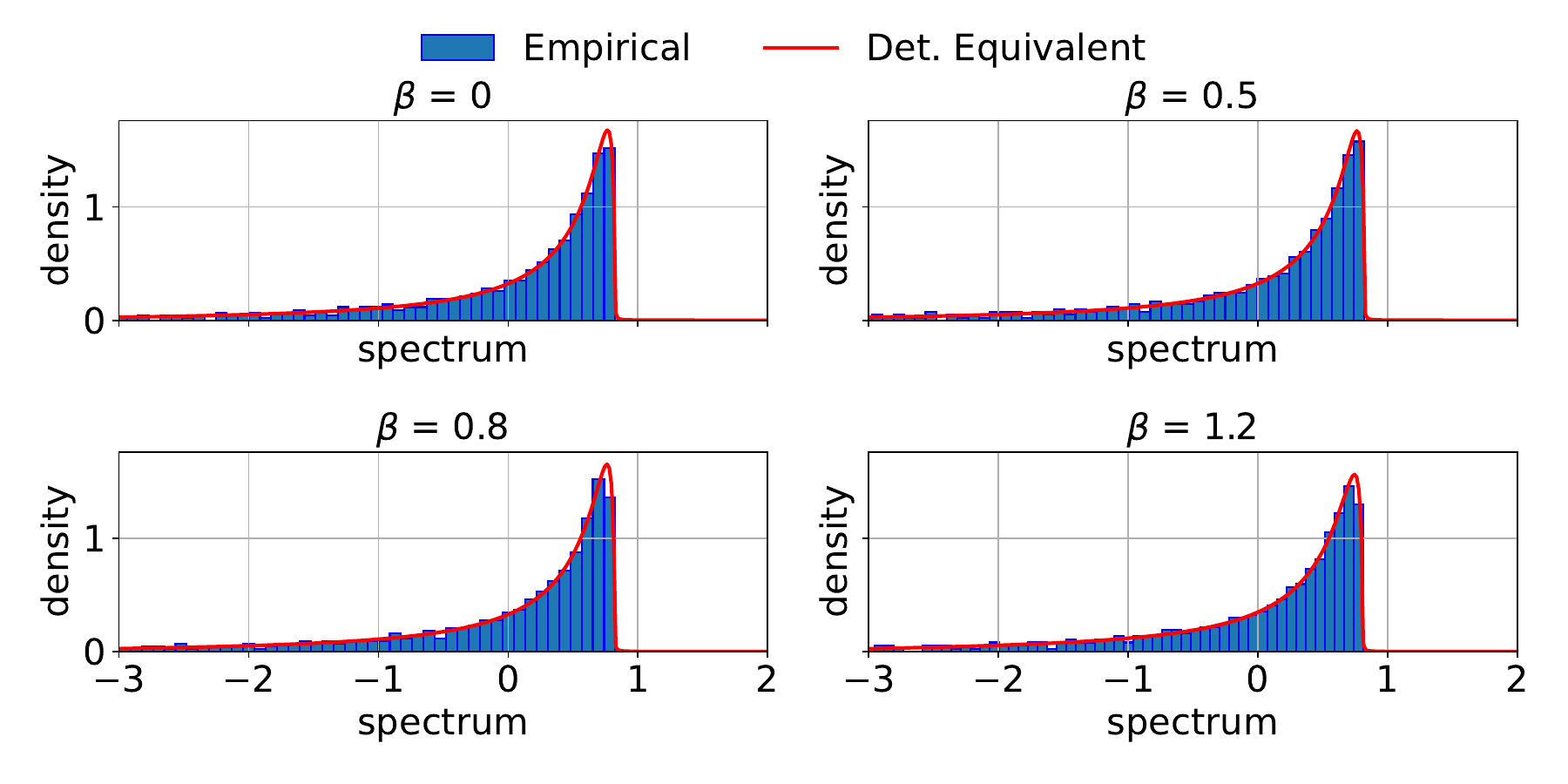}
        \label{fig:second}
    \end{subfigure}
    \caption{\textbf{Left:} Spectrum of the SSR matrix $\hat{A}$ for Gaussian Data with a Toeplitz Covariance parametrized by $\rho \in (0,1)$, that is: $\Sigma_{i,j} = \rho^{|i-j|}$. The blue lines are the empirical spectrum, while the red lines are predicted by \Cref{thm:det_equivalent_A_hat}. In both experiments $\lambda=0.01$, $d =1000$ and $\alpha = 3$. \textbf{Right: }Spectral density of $\hat{A}$ and the predicted spectral density from \Cref{corollary:universality} for Gaussian data with covariance $\Sigma = C_{\beta}\mathrm{diag}(1, 2^{-\beta}, \dots, d^{-\beta})$, with $C_{\beta} $ such that $\mathrm{Tr}(\Sigma) = 1$. In all plots $\lambda = 10^{-5}$ and $d = 500$.}
    \label{fig:spectrum_plots}
\end{figure}

 As noted in Remark~\ref{rmk:real_eigenvalues_A}, we can study the eigenvalues of $\hat{A}$ by studying the eigenvalues of the matrix $\Lambda^{\frac{1}{2}} Q(\lambda) \Lambda^{\frac{1}{2}}$, for $\Lambda = [\mathrm{diag}(Q(\lambda))]^{-1}$. Our second main result is a characterization
of the asymptotic spectrum of $\Lambda^{\frac{1}{2}} Q(\lambda) \Lambda^{\frac{1}{2}}$, and therefore that of the SSR predictor $\hat{A}$. 

\begin{theorem}
For $\lambda >0$ and $\alpha = \frac{n}{d}$, let $\tilde{m}(\lambda)$ be the unique solution of the self-consistent equation: 
\begin{equation}
    \tilde{m}(\lambda) = \left ( \lambda + \dfrac{1}{n} \mathrm{Tr}\left ( \Sigma (\tilde{m} \Sigma + I_{d})^{-1}\right )\right )^{-1}.
    \label{eq:self_consistent_m}
\end{equation}
For $z \in \C^{+}$, let $\chi$ be the solution to the self-consistent equation
\begin{equation}
    \chi = \frac 1 n \mathrm{Tr}\left(\Sigma(-\lambda I_d - \frac{1}{1-\chi} \Sigma + \bar{D}/z)^{-1}\right).
    \label{eq:chi_equation}
\end{equation}
where $\bar{D}$ is a diagonal matrix with entries $\bar{D}_{k,k}= \frac{\lambda}{(\tilde{m}(\lambda) \Sigma +  I_{d})^{-1}_{k,k}}$. Then, the matrix: 
\begin{equation}\label{eq:G_equiv}
    \mathcal{G}(z) = \left[\bar{D}^{\frac{1}{2}}\left(\frac{1}{1-\chi} \Sigma + \lambda I_d\right)^{-1} \bar{D}^{\frac{1}{2}} - z I_d\right]^{-1}
\end{equation}
is a deterministic equivalent for the resolvent:
\begin{equation}
    G(z) := (\Lambda^{\frac{1}{2}} Q(\lambda) \Lambda^{\frac{1}{2}} - zI_{d})^{-1}.
    \label{eq:resolvent_D_Q_D}
\end{equation}
As an immediate consequence, for any $z \in \C^+$, the Stieltjes transform of the empirical spectral measure of $\Lambda^{\frac{1}{2}} Q(\lambda) \Lambda^{\frac{1}{2}}$ evaluated at $z$ converges in probability to $\frac 1 d \mathrm{Tr} \,\mathcal{G}(z)$.
\label[theorem]{thm:det_equivalent_A_hat}
\end{theorem}

For the proof of \Cref{thm:det_equivalent_A_hat} we refer the reader to Appendix~\ref{section:appendix_spectrum}. The left of \Cref{fig:spectrum_plots} illustrates the empirical spectral density of $\hat{A}$ with the deterministic equivalent given in \Cref{thm:det_equivalent_A_hat} for Gaussian data generated from an AR(1) process. (See Section~\ref{sec:toeplitz} for details of this model.)

\subsubsection{Universality of the spectrum} 
An interesting consequence of \Cref{thm:det_equivalent_A_hat} arises when we focus on the case where the covariance $\Sigma$ is a diagonal matrix and we consider the ridge-less limit $\lambda\to 0^{+}$.  From \Cref{lemma:approximation_error}, we know that this task is equivalent to simply fitting noise.
\begin{corollary}
    Assume $\Sigma$ is a diagonal matrix. Then, for $\alpha >1$ and $\lambda \to 0$, the deterministic equivalent for the resolvent in \Cref{thm:det_equivalent_A_hat} is independent of $\Sigma$: 
    \[
     \mathcal{G}(z) = \dfrac{1}{(1 - \frac{1}{\alpha} ) (1 - \chi) - z} I_d,
    \]
    where $\chi$ solves: 
    \begin{equation}
        \chi ^2 + (z-1) \chi + \dfrac{z}{\alpha -1} = 0.
        \label{eq:chi_self_consistent_diag}
    \end{equation}
    In particular, the spectral density of the SSR matrix $\hat{A}$ is supported in the interval $\left[\frac{-2}{\sqrt{\alpha} -1}, \frac{2}{\sqrt{\alpha} +1}\right]$.
    \label[corollary]{corollary:universality}
\end{corollary}
The fact that $\mathcal{G}(z)$ is independent of the values in $\Sigma$ means the asymptotic spectrum in the ridge-less limit is universal. The right of \Cref{fig:spectrum_plots} illustrates this universality for the family of diagonal power-law covariances parametrized by an exponent $\beta >0$. Despite the heavy-tailed spectrum of the population covariance, as $\lambda$ approaches $0^{+}$ the different spectra collapse to the same universal shape for $\alpha >1$. 
As a sanity check for our results on the spectrum of the SSR matrix, consider the simplest possible setting of i.i.d. isotropic Gaussian data  $x_i \sim \mathcal{N}(0,I_{d})$. In this case, the asymptotic spectral density of $\hat{\Sigma}_{n} = \frac{1}{n}X^T X$ is the Marchenko-Pastur law,  $\tilde{m}(\lambda)$ is the solution of 
\begin{equation}
    - \frac{\lambda}{\alpha} \tilde{m}(\lambda)^2  - ( 1- \frac{1}{\alpha} + \lambda) \tilde{m}(\lambda)  + 1 = 0. \label{eq:steltjes_transform_eq_isotropic}
\end{equation}
Then in this case, the spectral density of the SSR matrix $\hat{A}$ is just a transformation of the Marchenko-Pastur distribution.  To illustrate this, \Cref{fig:spectrum_isotropic} shows the empirical spectral density of $\hat{A}$ with the deterministic spectral density stated in \Cref{thm:det_equivalent_A_hat}. In the case where $\alpha <1 $, there are fluctuations around the atom on the left. These fluctuations decrease as $d$ goes to infinity.

\begin{figure}[t] 
    \centering
    \includegraphics[width=0.7\linewidth]{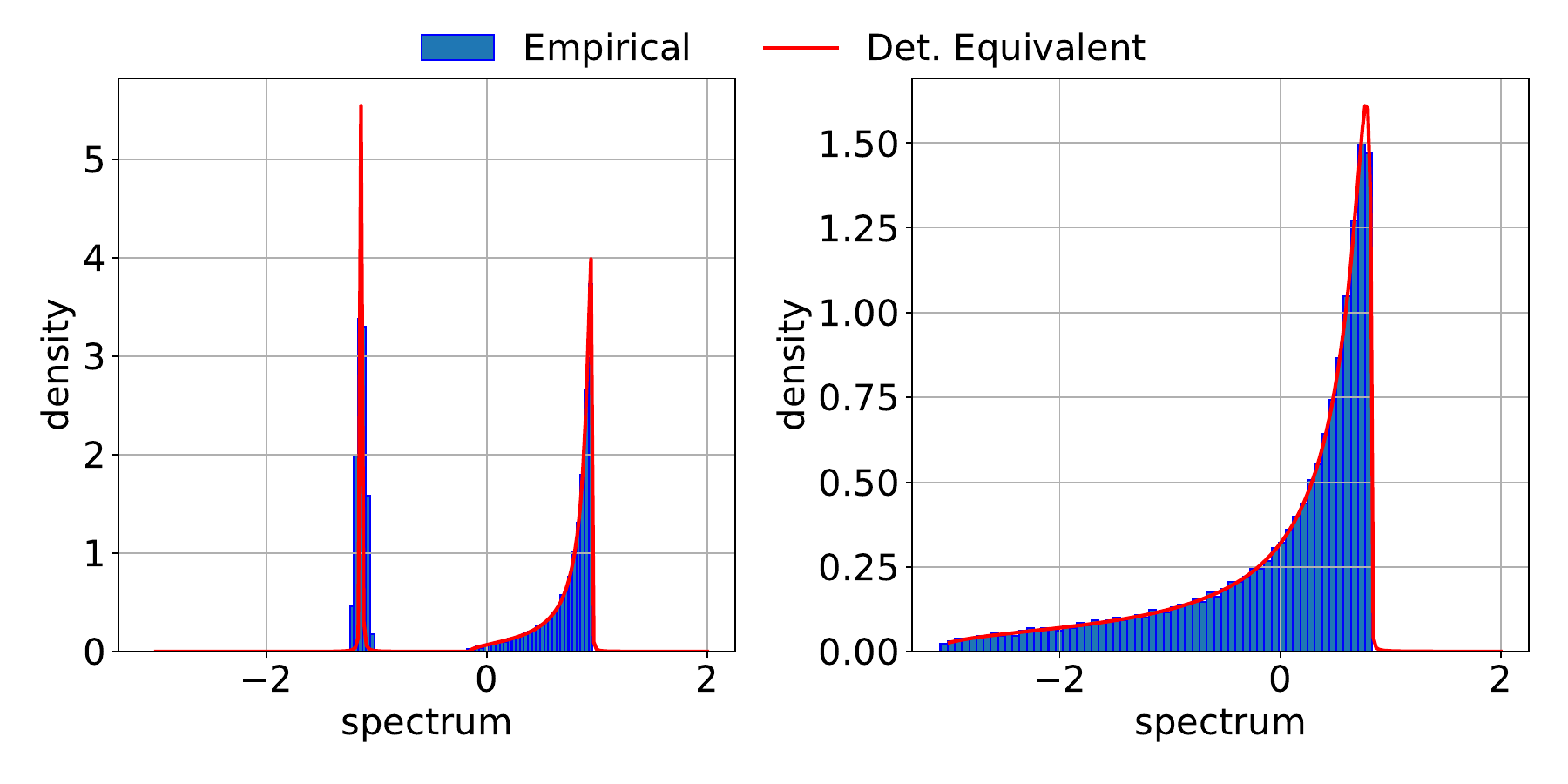}
    \caption{Empirical spectral density of $\hat{A}$ for Gaussian, isotropic data, compared with the spectrum predicted by \Cref{thm:det_equivalent_A_hat}. The dimension is $d=2000$ and $\lambda = 0.01$. On the left, $\alpha = 0.6$, and on the right, $\alpha = 1.5$.}
    \label{fig:spectrum_isotropic}
\end{figure}
\section{Case Studies}
In the last section, we showed that the high-dimensional limit of the SSR predictor can be characterized by a deterministic equivalent depending on the data only through the properties of the population covariance $\Sigma$. We now turn to an in-depth discussion of two particular examples of structured data which are popular in the machine learning literature, corresponding to two different types of tasks: one in which the coordinates of $x$ are correlated through an underlying low-dimensional structure, and one in which $x$ follows an autoregressive structure.  Since we are considering a linear SSL, a natural linear unsupervised learning benchmark is PCA. For the reader's convenience, we recall that the PCA estimator with $p$ principal components (or simply, $p$ directions) is defined as: 
\begin{equation}
    A^\mathrm{PCA}_{p} = \arg \min_{A \succ 0, \mathrm{rank}(A) = p} \dfrac{1}{dn} \sum_{i=1}^{n} \| x_i - Ax_i \|^2.
    \label{eq:def_PCA}
\end{equation}

\subsection{Spiked Covariance}

\begin{figure}[t]
    \centering
    \begin{subfigure}{0.48\textwidth}
        \centering
        \includegraphics[width=\linewidth]{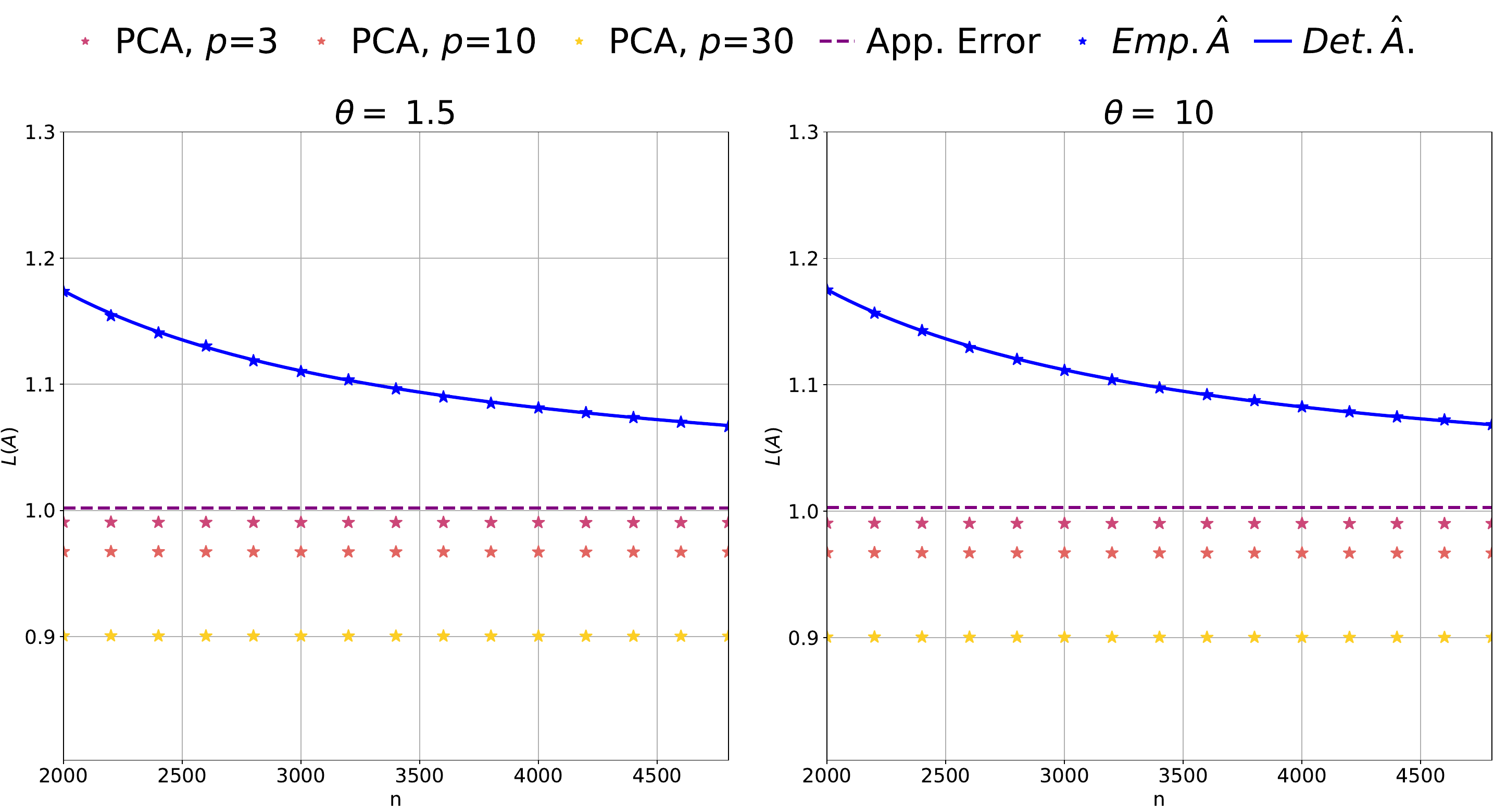}
    \end{subfigure}
    \hfill
    \begin{subfigure}{0.48\textwidth}
        \centering
        \includegraphics[width=\linewidth]{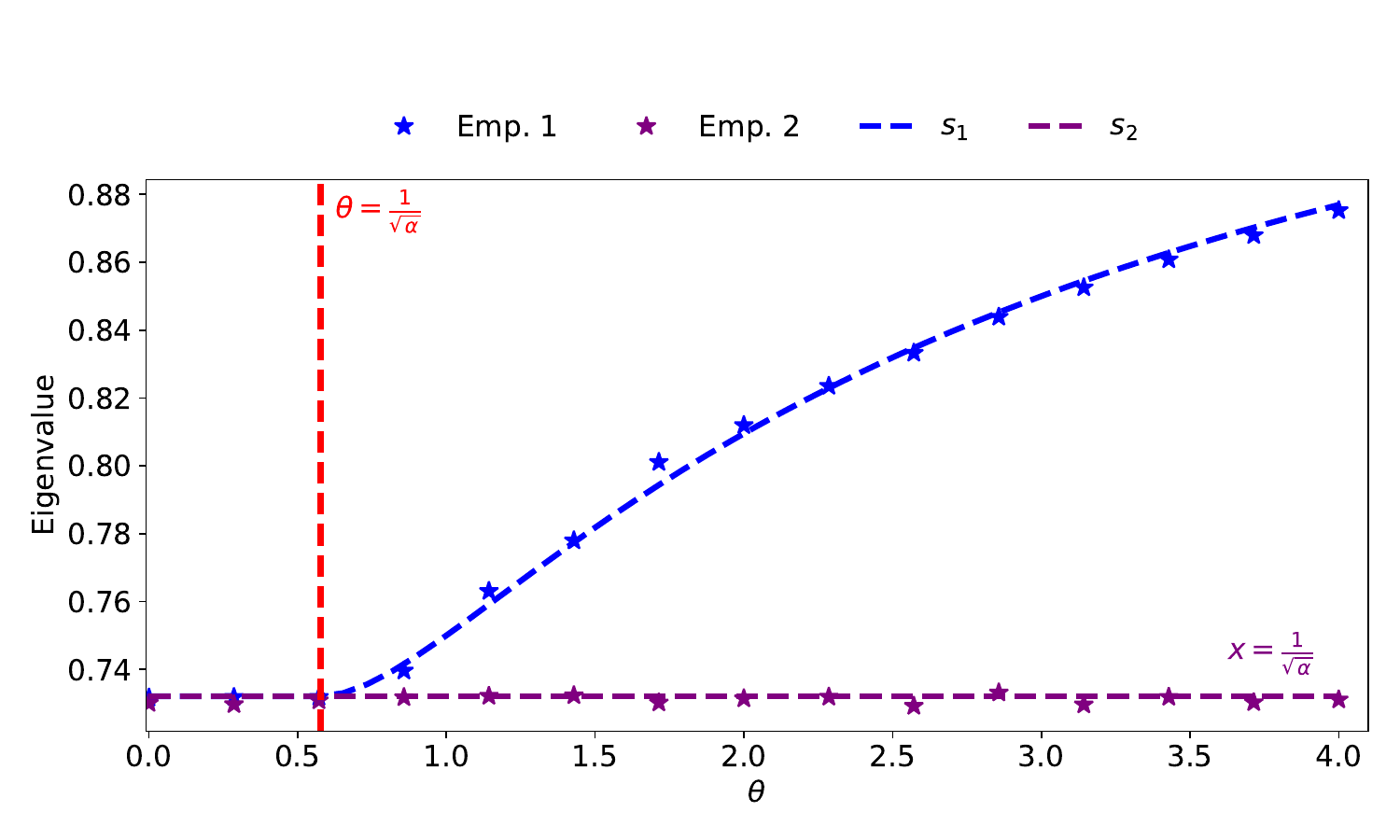}
    \end{subfigure}
    \caption{\textbf{Left:} Comparison of the SSR estimator and PCA for a spiked covariance $\Sigma = I_{d} + \theta vv^T$, for $v \sim \mathrm{Unif}(\mathbb{S}^{d-1})$. In the experiment, $d=300,\lambda = 0.01$ and PCA is applied for $p \in \{ 3, 10,30\}$. \textbf{Right: } Baik-Ben Arous-Peché transition for the SSR matrix $\hat{A}$ for Gaussian data with covariance $\Sigma = I_{d} + \theta vv^T$, for $v \sim \mathrm{Unif}(\mathbb{S}^{d-1})$ and varying $\theta$. The value of $\lambda = 10^{-5}$, $d = 2000$, and $\alpha = \frac{n}{d} = 2$.}
    \label{fig:pca_spiked_covariance}
\end{figure}

In this section, we consider the case where $\Sigma = I_{d} + \theta v v^{\top}$, with $v \in \mathbb{S}^{d-1}$, and $\theta \geq 0$. Known as the \emph{spiked covariance model}, it is a popular model in the statistics literature for a situation where one needs to learn a signal $v \in \RR^{d}$ in otherwise noisy data  \citep{donoho2018optimal}.

The first question we want to answer in this setting is:  What is the performance of $\hat{A}$ compared to PCA in the task of reconstructing a random vector with covariance given by $\Sigma = I_{d} + \theta vv^{\top}$? How does it change depending on the signal-to-noise ratio $\theta$? We first answer this question in the classical limit $n\to\infty$. 
\begin{lemma}\label[lemma]{lemma:pca_comparison_spiked_cov}
    Consider centered data with distribution satisfying the assumptions of \Cref{theorem:det_equivalent_gen_error}, and covariance $\Sigma = I_{d} + \theta vv^{\top}$, for $\theta \geq 0$ and $v \in \mathbb{S}^{d-1}$. Let $L(\mathrm{SSR})$ denote the infinite-data generalization error in \cref{eq:gen_error_expression} for Self-Supervised Ridge Regression, and $L(\mathrm{PCA}_{p})$ denote the infinite-data generalization error for PCA with $p$ directions. Then, for any values of $\theta, \lambda$ and $p \geq 1$, we have: 
    \[
    L(\mathrm{PCA}_{p}) < L(\mathrm{SSR}). 
    \]
\end{lemma}
This lemma gives us insight on how strong the correlation in the data needs to be in order for SSR to achieve a reasonable performance. Recall that for isotropic data $\Sigma = I_{d}$ the best SSR matrix is $\hat{A}=0$, since there is no correlation in the data. We can see the spike $\theta vv^{\top}$ as a perturbation of this case, which introduces correlations between the entries of $x$. Indeed, for a constant SNR $\theta=\Theta_{d}(1)$, the correlations are vanishing in $d$ (because $\| v\|_2 =1$) so the correlations are rather weak in the high-dimensional limit. \Cref{fig:pca_spiked_covariance} illustrates this by comparing the empirical error of PCA with the asymptotic limit of the generalization error in \Cref{theorem:det_equivalent_gen_error} for different values of the SNR $\theta$ in the large $n$ limit. As it can be seen, even the approximation error in this case is not better than $L(\hat{A}) = 1$, which corresponds to no reconstruction.  

We now turn to the question of whether signatures of the low-dimensional structure of $\Sigma$ can be detected in the spectrum of the SSR matrix $\hat{A}$. Recall that  the sample covariance matrix $\hat{\Sigma}_{n}$ in this model undergoes a Baik-Ben Arous-Peché (BBP) transition \citep{BBAP05}: the spectrum of $\hat{\Sigma}$ will not identify the spiked direction $v$ unless $\theta>\theta^\star$, for some critical value $\theta^{\star}>0$. In other words, under certain values of $\theta$ the spectral distribution of $\hat{\Sigma}$ is no different than if the model was sampled from pure noise (isotropic data). As $\theta$ is increased, there exists a value $\theta^\star$ after which an isolated eigenvalue pops out of the Marchenko-Pastur bulk of $\hat{\Sigma}_{n}$.

Since the matrix $\hat{A}$ is related to the sample covariance through the resolvent $Q(\lambda)$, a natural question is whether a similar BBP transition can be observed for $\hat{A}$.  Building on the deterministic equivalents found in\ref{thm:det_equivalent_A_hat} and \Cref{corollary:universality}, we prove the following result concerning this phase transition. 

\begin{proposition}
Let $s_1, s_2$ denote the top two eigenvalues of the SSR matrix $\hat A$ in the ridgeless limit, and let $\xi$ be the solution of \Cref{eq:chi_self_consistent_diag}. Then
\begin{equation}
    s_1 \xrightarrow[]{n, d \to \infty}\begin{cases}
        1 - z^\ast &\text{if } \theta  > \frac{1}{\sqrt \alpha}\\
        \frac{2}{1 + \sqrt \alpha} &\text{otherwise },
    \end{cases}
\end{equation}
where $z^\ast$ solves equation: 
\begin{equation}\label{eq:z_equation}
    \frac{z\alpha}{(\alpha-1)(1-\chi)} = \frac{1}{1 +\theta},
\end{equation}
for $\chi$ solving \cref{eq:self_consistent_equation_diagonal}, 
and $ s_2 \xrightarrow[]{n, d \to \infty} \frac{2}{1 + \sqrt \alpha}$.
\label[proposition]{prop:BBP_transition}
\end{proposition}

\begin{remark}
The classical BBP transition for the covariance also occurs at $\theta^*= \sqrt{\frac{d}{n}} = \sqrt{\frac{1}{\alpha}}$, so the phase transition occurs at the same SNR in both models. The second eigenvalue $s_2$ converges to the top of the bulk in the case where $\Sigma = I_d$ (see \Cref{corollary:universality}). 
\end{remark}

The result in \Cref{prop:BBP_transition} is illustrated in \Cref{fig:pca_spiked_covariance}.

\subsection{Auto-Regressive Model of Order 1}
\label{sec:toeplitz}
\begin{figure}[t]
    \centering
    \begin{subfigure}{0.48\textwidth}
        \centering
        \includegraphics[width=\linewidth]{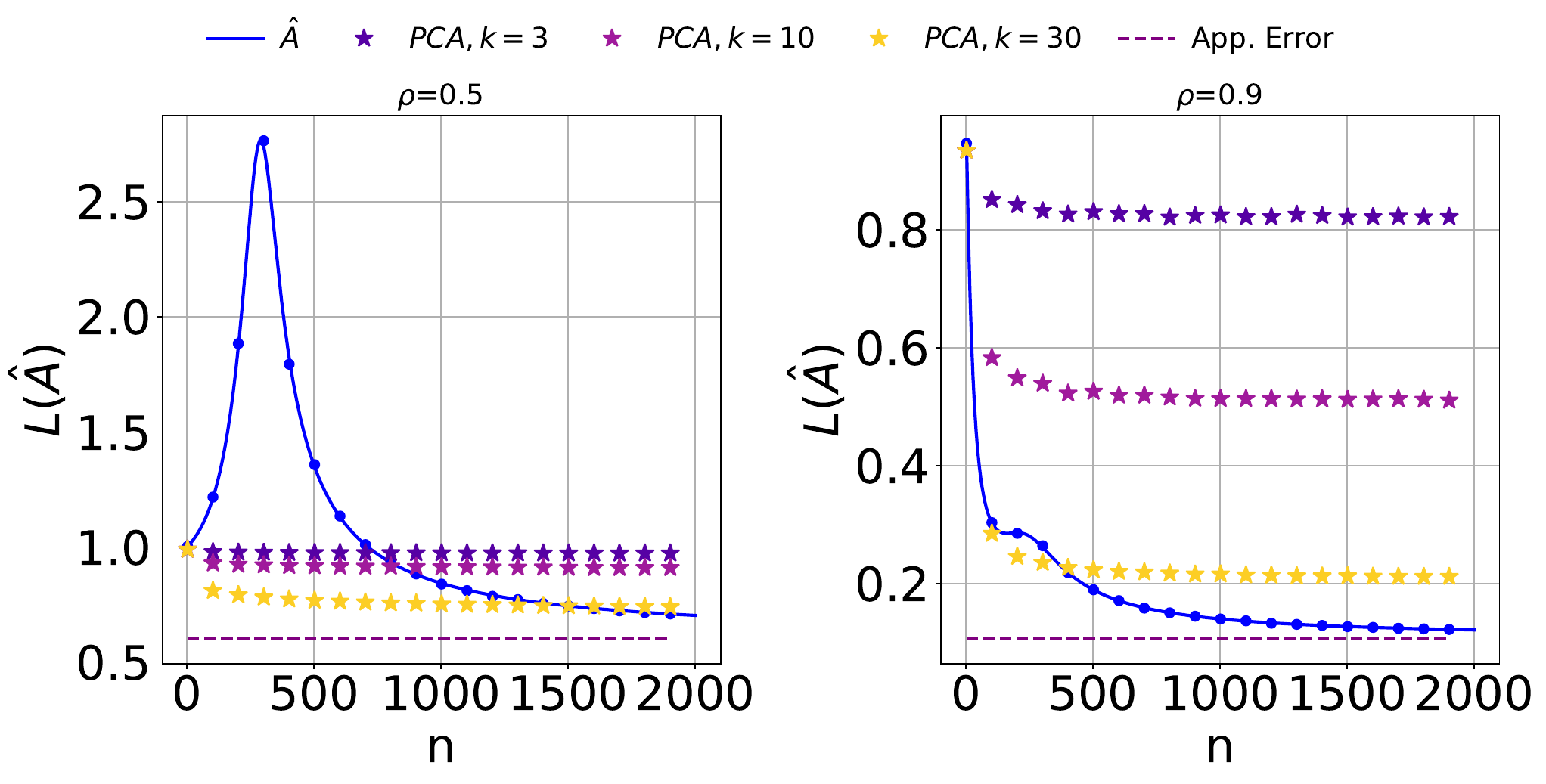}
    \end{subfigure}
    \hfill
    \begin{subfigure}{0.48\textwidth}
        \centering
        \includegraphics[width=\linewidth]{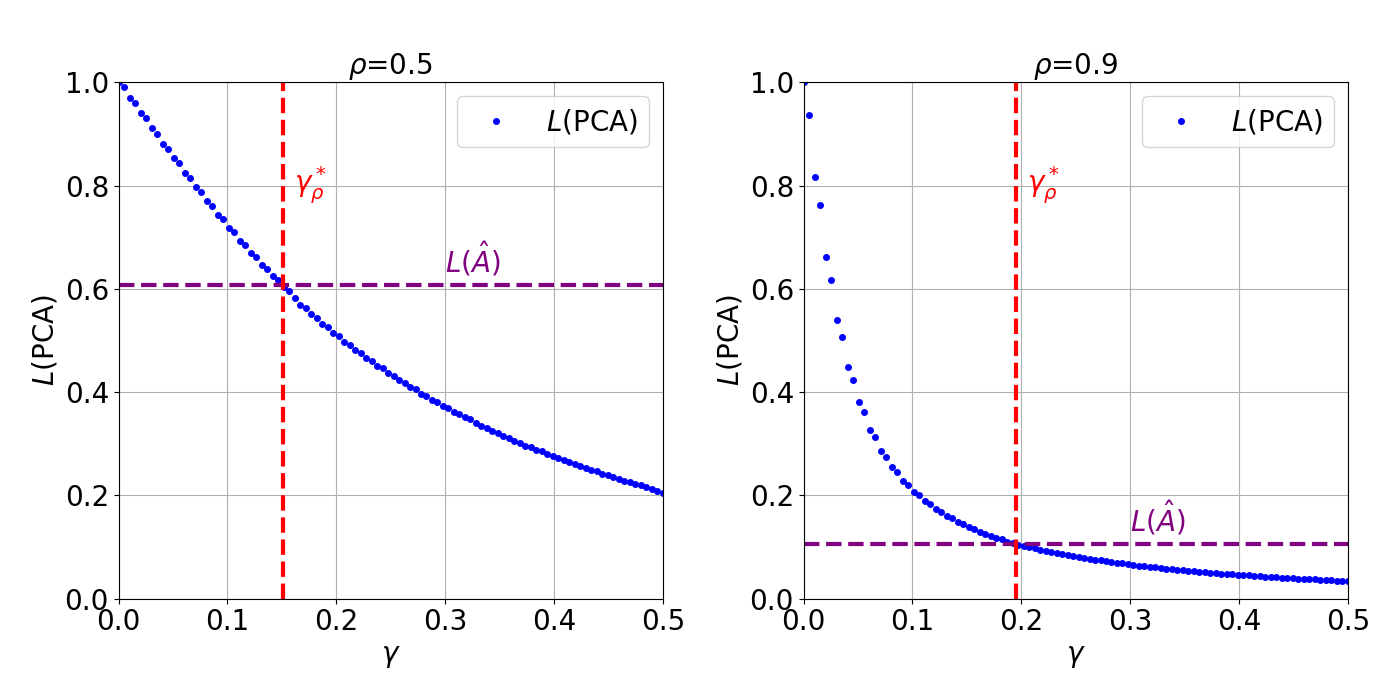}
    \end{subfigure}
    \caption{\textbf{Left:} Empirical generalization error for PCA and the self-supervised ridge estimator for Gaussian data with a Toeplitz covariance, for different values of $\rho$. In both pictures, $d=200$ and $\lambda=0.01$. \textbf{Right: } Generalization error for PCA and the self-supervised Ridge estimator for Gaussian data with a Toeplitz covariance, for different values of $\rho$. The sample size is fixed at $n=20.000$, the dimension is $d=300$ and $\lambda = 10^{-5}$. For PCA with $p$ directions, $\gamma = \frac{p}{d}$.}
    \label{fig:pca_toeplitz}
\end{figure}

As a second example, we now consider a case in which $x$ follows an auto-regressive structure common in sequence modeling. The simplest such example is stationary autoregressive Process of order one (a.k.a. AR(1) process):
\begin{equation}
x_{k} = \rho x_{k-1} + \varepsilon_{k}, \quad k \in [d],  
\label{eq:ar_1_model}
\end{equation}
where $\varepsilon_{k}$ are i.i.d centered random variables with variance $\frac{1}{1-\rho^2}$ so that ${\rm Var}(x_{k}) = 1$, and $0 < \rho < 1$ is a parameter. Note that the $\rho \to 0^{+}$ limit corresponds to the isotropic setting. 

The covariance of AR(k) processes are Toeplitz matrices (see, e.g Appendix 3 in \citep{PB20}): A matrix $\Sigma \in \RR^{d \times d}$ such that its entry $\Sigma_{jk}$ depends on the difference $|j - k|$. In the AR(1) case: 
\begin{equation}
    \Sigma_{jk} = \EE \left[x_{j}x_{k}\right] = \rho^{|j-k|},
    \label{eq:ar_1_toeplitz_cov}
\end{equation}
which describes the correlation between $x_{j}$ and $x_{k}$ that decays exponentially as $|j-k|$ grows. The spectrum and properties of this matrix have been studied, for example, in \citep{KMS53} and \citep{NS20}. Our first result for this setting shows that unlike the spiked covariance case, in the population limit SSR outperforms PCA unless the number of directions $p=\Theta(d)$. This is summarized in the following proposition. 

\begin{proposition}
Consider centered data with distribution satisfying the assumptions of \Cref{theorem:det_equivalent_gen_error}, and covariance of the form in \Cref{eq:ar_1_toeplitz_cov}, with a fixed parameter $\rho \in (0,1)$.  Let $L(\mathrm{SSR})$ and $L(\mathrm{PCA}_{p})$ denote the population limit of the generalization error for self-supervised ridge regression and PCA with $p$ directions, respectively. Let $\gamma:=\frac{p}{d}$. Then, as $d$ grows to infinity and $\lambda \to 0^{+}$, we have 
\begin{equation}
L(\mathrm{PCA}_{p}) < L (\mathrm{SSR}),    
\end{equation}
if and only if
\[
    \dfrac{2}{\pi}\arctan \left ( \frac{1-\rho}{1+\rho}\tan \left (  \frac{\pi}{2} \left (\dfrac{(2\rho)^2}{(1+\rho)^2 + (1-\rho)^2} \right )\right ) \right ) \leq \gamma.
\]
\label[proposition]{prop:pca_comparison_toeplitz}
\end{proposition}

\begin{figure}[t]
    \centering
    \includegraphics[width=0.4\linewidth]{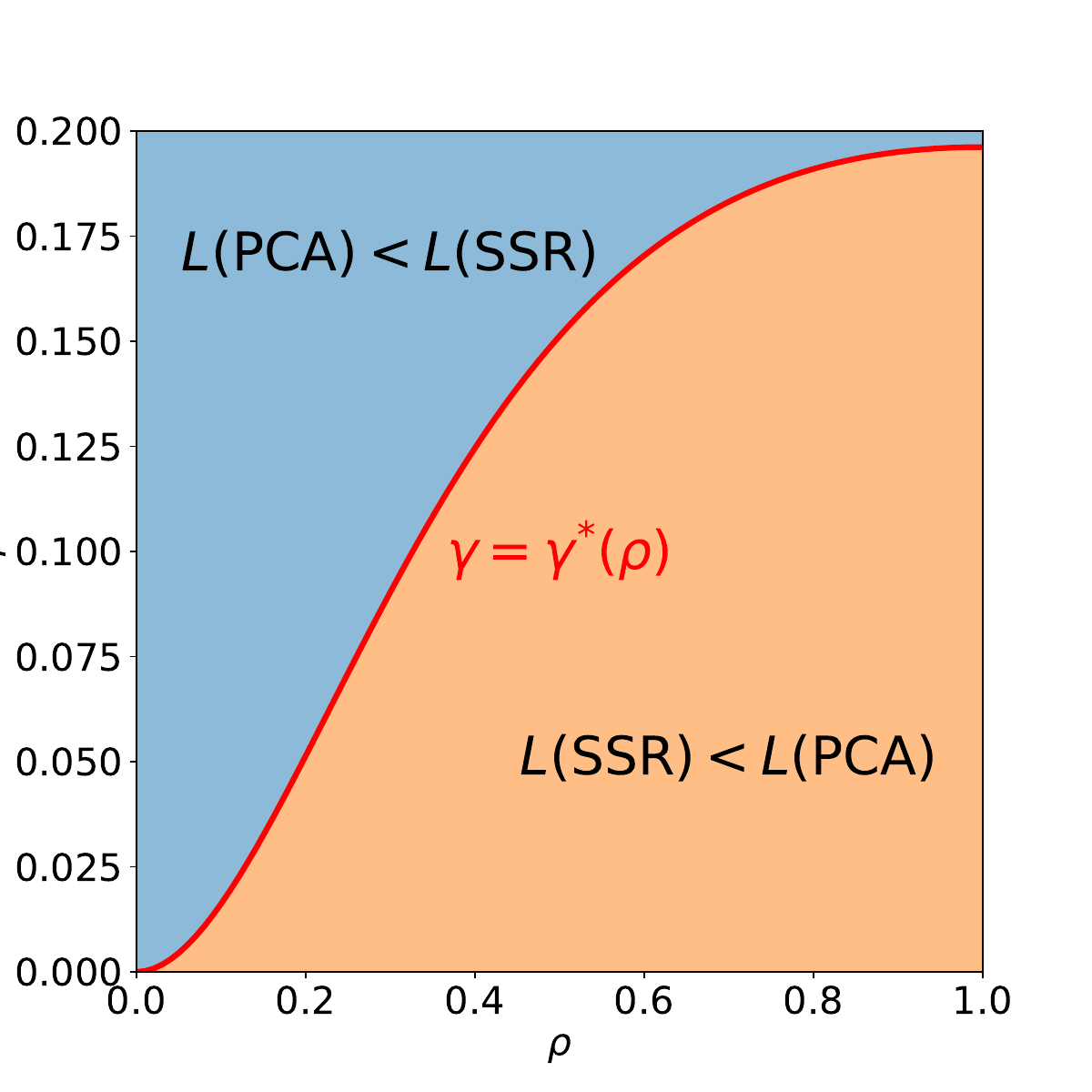}
    \caption{The phase transition curve defined by \Cref{eq:phase_transition_curve}. Above the curves, we have that PCA achieves a better generalization error than the SSR predictor. Below the curve, the SSR predictor achieves a better generalization error. }
    \label{fig:phase_transition_boundary}
\end{figure}

Note that the LHS in the inequality in \Cref{prop:pca_comparison_toeplitz} is increasing in $\rho$. Therefore, as the SNR $\rho$ in the AR(1) model in \Cref{eq:ar_1_model} becomes bigger, it becomes harder for PCA to outperform the self-supervised estimator (as measured by training sample-complexity). Moreover, in the population limit, the number of directions has to be extensive in $d$ in order to reach the same performance the SSR predictor. Figure~\ref{fig:phase_transition_boundary} illustrates the curve $\gamma^{\star}(\rho)$.  

\Cref{prop:pca_comparison_toeplitz} is also illustrated in \Cref{fig:pca_toeplitz}. The Figure in the left compares these two estimators for increasing values of $\rho$ in \Cref{eq:ar_1_model}, corresponding to an increasing dependence of $x_k$ on $x_{k-1}$. As the sample complexity $\alpha=\sfrac{n}{d}$ grows, we see that for PCA to be at the same generalization error as the self-supervised estimator, a greater number of directions $p$ is needed. Its also possible to see that for higher $\rho$ the peak of $L(\hat{A})$ is also lower. This is explained by the fact that, as computed in Appendix~\ref{appendix_toeplitz_finite_alpha}, for high-dimensional Toeplitz matrices, the degrees of freedom are approximately: 
\begin{equation}
\mathrm{df}^{\Sigma}_2(\kappa) \approx \dfrac{\sqrt{1- \rho^{2}}}{4\sqrt{\kappa }}.
\end{equation}
In the right of \Cref{fig:pca_toeplitz},  $\gamma^{\star}_\rho$ denotes the quantity on the LHS of \Cref{prop:pca_comparison_toeplitz}. For both values of $\rho$ in the figure, $L(\mathrm{SSR}) < L(\mathrm{PCA})$, approximately until $\gamma$ reaches $\gamma^{*}_\rho$.

\section{Conclusion}

In this work, we studied masked self supervised models in their simplest declination: self-supervised ridge regression. We derived asymptotic limits for the generalization error and the asymptotic spectral distribution of the matrix-valued SSR estimator. With this characterization, we proved the importance of structure in the data for this type of tasks, and tested this on two classical statistical models: the spiked covariance model and an auto-regressive process of order $1$. 

\section*{Acknowledgements}
We would like to thank Lenka Zdeborov\'a for insightful discussions. This research was motivated by discussions during the ``Huddle on Learning and Inference from Structured Data'' at ICTP
Trieste in 2023. We would like to thank the huddle organizers Jean Barbier, Manuel S\'aenz, Subhabrata Sen, and Pragya Sur for their hospitality and many helpful discussion. BL and AW were supported by the French government, managed by the National Research Agency (ANR), under the France 2030 program with the reference ``ANR-23-IACL-0008'' and the Choose France - CNRS AI Rising Talents program. AW was also funded by the PSL Graduate Program in Computer Science. The work of YML is supported by a Harvard College Professorship, by the Harvard FAS Dean's Fund for Promising Scholarship, and by DARPA under grant DIAL-FP-038. FG is supported by European Union-NextGenerationEU
(NGEU) and she is partially supported by project SERICS (PE00000014) under the MUR National Recovery and Resilience Plan.

\bibliography{bibliography}

@book{CL22,
  title={Random matrix methods for machine learning},
  author={Couillet, Romain and Liao, Zhenyu},
  year={2022},
  publisher={Cambridge University Press}
}

@article{V10,
  title={Introduction to the non-asymptotic analysis of random matrices},
  author={Vershynin, Roman},
  journal={arXiv preprint arXiv:1011.3027},
  year={2010}
}

@article{mei2022generalization,
  title={Generalization error of random feature and kernel methods: Hypercontractivity and kernel matrix concentration},
  author={Mei, Song and Misiakiewicz, Theodor and Montanari, Andrea},
  journal={Applied and Computational Harmonic Analysis},
  volume={59},
  pages={3--84},
  year={2022},
  publisher={Elsevier}
}

@inproceedings{devlin2019bert,
  title={Bert: Pre-training of deep bidirectional transformers for language understanding},
  author={Devlin, Jacob and Chang, Ming-Wei and Lee, Kenton and Toutanova, Kristina},
  booktitle={Proceedings of the 2019 conference of the North American chapter of the association for computational linguistics: human language technologies, volume 1 (long and short papers)},
  pages={4171--4186},
  year={2019}
}

@article{liu2019roberta,
  title={Roberta: A robustly optimized bert pretraining approach},
  author={Liu, Yinhan and Ott, Myle and Goyal, Naman and Du, Jingfei and Joshi, Mandar and Chen, Danqi and Levy, Omer and Lewis, Mike and Zettlemoyer, Luke and Stoyanov, Veselin},
  journal={arXiv preprint arXiv:1907.11692},
  year={2019}
}

@article{B24,
  title={High-dimensional analysis of double descent for linear regression with random projections},
  author={Bach, Francis},
  journal={SIAM Journal on Mathematics of Data Science},
  volume={6},
  number={1},
  pages={26--50},
  year={2024},
  publisher={SIAM}
}

@book{PB20,
  title={A first course in random matrix theory: for physicists, engineers and data scientists},
  author={Potters, Marc and Bouchaud, Jean-Philippe},
  year={2020},
  publisher={Cambridge University Press}
}

@article{NS20,
  title = {The Toeplitz matrix {$e^{-\kappa |i-j|}$} and its application to a layered electron gas},
  author = {Narayan, Onuttom and Shastry, B. Sriram},
  journal = {arXiv preprint arXiv:2006.15436},
  year = {2020}
}

@article{KMS53,
  title={On the eigenvalues of certain Hermitian forms},
  author={Kac, Marek and Murdock, WL and Szeg{\"o}, Gabor},
  journal={Journal of Rational Mechanics and Analysis},
  volume={2},
  pages={767--800},
  year={1953},
  publisher={JSTOR}
}

@article{RW20,
  title={Nonasymptotic upper bounds for the reconstruction error of PCA},
  author={Reiss, Markus and Wahl, Martin},
  journal={The Annals of Statistics},
  volume={48},
  number={2},
  pages={1098--1123},
  year={2020},
  publisher={JSTOR}
}

@inproceedings{EHEM25,
  title={A Geometric Analysis of PCA},
  author={El Hanchi, Ayoub and Erdogdu, Murat A and Maddison, Chris J},
  booktitle={The Thirty-ninth Annual Conference on Neural Information Processing Systems},
  year={2025}
}

@article{BBAP05,
  title={Phase transition of the largest eigenvalue for nonnull complex sample covariance matrices},
  author={Baik, Jinho and Ben Arous, G{\'e}rard and Péché, Sandrine},
  journal={Annals of probability},
  volume={33},
  number={5},
  pages={1643--1697},
  year={2005}
}

@article{CM24,
  title={Dimension free ridge regression},
  author={Cheng, Chen and Montanari, Andrea},
  journal={The Annals of Statistics},
  volume={52},
  number={6},
  pages={2879--2912},
  year={2024},
  publisher={Institute of Mathematical Statistics}
}

@article{DW18,
  title={High-dimensional asymptotics of prediction: Ridge regression and classification},
  author={Dobriban, Edgar and Wager, Stefan},
  journal={The Annals of Statistics},
  volume={46},
  number={1},
  pages={247--279},
  year={2018},
  publisher={JSTOR}
}

@article{rende2024mapping,
  title={Mapping of attention mechanisms to a generalized potts model},
  author={Rende, Riccardo and Gerace, Federica and Laio, Alessandro and Goldt, Sebastian},
  journal={Physical Review Research},
  volume={6},
  number={2},
  pages={023057},
  year={2024},
  publisher={APS}
}

@inproceedings{he2022masked,
  title={Masked autoencoders are scalable vision learners},
  author={He, Kaiming and Chen, Xinlei and Xie, Saining and Li, Yanghao and Doll{\'a}r, Piotr and Girshick, Ross},
  booktitle={Proceedings of the IEEE/CVF conference on computer vision and pattern recognition},
  pages={16000--16009},
  year={2022}
}

@article{bao2021beit,
  title={Beit: Bert pre-training of image transformers},
  author={Bao, Hangbo and Dong, Li and Piao, Songhao and Wei, Furu},
  journal={arXiv preprint arXiv:2106.08254},
  year={2021}
}

@article{donoho2018optimal,
  title={Optimal shrinkage of eigenvalues in the spiked covariance model},
  author={Donoho, David L and Gavish, Matan and Johnstone, Iain M},
  journal={Annals of statistics},
  volume={46},
  number={4},
  pages={1742},
  year={2018}
}

@inproceedings{schroder2023deterministic,
  title={Deterministic equivalent and error universality of deep random features learning},
  author={Schr{\"o}der, Dominik and Cui, Hugo and Dmitriev, Daniil and Loureiro, Bruno},
  booktitle={International Conference on Machine Learning},
  pages={30285--30320},
  year={2023},
  organization={PMLR}
}

@InProceedings{schroder24asymptotics,
  title = 	 {Asymptotics of Learning with Deep Structured ({R}andom) Features},
  author =       {Schr\"{o}der, Dominik and Dmitriev, Daniil and Cui, Hugo and Loureiro, Bruno},
  booktitle = 	 {Proceedings of the 41st International Conference on Machine Learning},
  pages = 	 {43862--43894},
  year = 	 {2024},
  editor = 	 {Salakhutdinov, Ruslan and Kolter, Zico and Heller, Katherine and Weller, Adrian and Oliver, Nuria and Scarlett, Jonathan and Berkenkamp, Felix},
  volume = 	 {235},
  series = 	 {Proceedings of Machine Learning Research},
  month = 	 {21--27 Jul},
  publisher =    {PMLR},
}

@InProceedings{dandi25random,
  title = 	 {A Random Matrix Theory Perspective on the Spectrum of Learned Features and Asymptotic Generalization Capabilities},
  author =       {Dandi, Yatin and Pesce, Luca and Cui, Hugo and Krzakala, Florent and Lu, Yue and Loureiro, Bruno},
  booktitle = 	 {Proceedings of The 28th International Conference on Artificial Intelligence and Statistics},
  pages = 	 {2224--2232},
  year = 	 {2025},
  editor = 	 {Li, Yingzhen and Mandt, Stephan and Agrawal, Shipra and Khan, Emtiyaz},
  volume = 	 {258},
  series = 	 {Proceedings of Machine Learning Research},
  month = 	 {03--05 May},
  publisher =    {PMLR},
}

@article{knowles2017anisotropic,
  title={Anisotropic local laws for random matrices},
  author={Knowles, Antti and Yin, Jun},
  journal={Probability Theory and Related Fields},
  volume={169},
  number={1},
  pages={257--352},
  year={2017},
  publisher={Springer}
}

@article{burda2004signal,
  title={Signal and noise in correlation matrix},
  author={Burda, Zdzislaw and G{\"o}rlich, A and Jarosz, Andrzej and Jurkiewicz, Jerzy},
  journal={Physica A: Statistical Mechanics and its Applications},
  volume={343},
  pages={295--310},
  year={2004},
  publisher={Elsevier}
}

@article{marvcenko1967distribution,
  title={Distribution of eigenvalues for some sets of random matrices},
  author={Mar{\v{c}}enko, Vladimir A and Pastur, Leonid Andreevich},
  journal={Mathematics of the USSR-Sbornik},
  volume={1},
  number={4},
  pages={457},
  year={1967},
  publisher={IOP Publishing}
}

@article{bai2008large,
  title={Large sample covariance matrices without independence structures in columns},
  author={Bai, Zhidong and Zhou, Wang},
  journal={Statistica Sinica},
  pages={425--442},
  year={2008},
  publisher={JSTOR}
}

@article{louart2018concentration,
  title={Concentration of measure and large random matrices with an application to sample covariance matrices},
  author={Louart, Cosme and Couillet, Romain},
  journal={arXiv preprint arXiv:1805.08295},
  year={2018}
}

@article{chouard2022quantitative,
  title={Quantitative deterministic equivalent of sample covariance matrices with a general dependence structure},
  author={Chouard, Cl{\'e}ment},
  journal={arXiv preprint arXiv:2211.13044},
  year={2022}
}

@article{vigeant2026dyson,
author = {Latourelle-Vigeant, Hugo and Paquette, Elliot},
title = {Dyson equation for correlated linearizations and test error of random features regression},
journal = {Random Matrices: Theory and Applications},
volume = {15},
number = {01},
pages = {2550026},
year = {2026},
doi = {10.1142/S2010326325500261},
}

@article{ilbert2024analysing,
  title={Analysing multi-task regression via random matrix theory with application to time series forecasting},
  author={Ilbert, Romain and Tiomoko, Malik and Louart, Cosme and Odonnat, Ambroise and Feofanov, Vasilii and Palpanas, Themis and Redko, Ievgen},
  journal={Advances in Neural Information Processing Systems},
  volume={37},
  pages={115021--115057},
  year={2024}
}

@article{pennington2017nonlinear,
  title={Nonlinear random matrix theory for deep learning},
  author={Pennington, Jeffrey and Worah, Pratik},
  journal={Advances in neural information processing systems},
  volume={30},
  year={2017}
}

@inproceedings{xiao2022precise,
  title={Precise learning curves and higher-order scaling limits for dot product kernel regression},
  author={Xiao, Lechao and Hu, Hong and Misiakiewicz, Theodor and Lu, Yue M and Pennington, Jeffrey},
  booktitle={Thirty-sixth Conference on Neural Information Processing Systems (NeurIPS)},
  year={2022}
}

@inproceedings{liao2018spectrum,
  title={On the spectrum of random features maps of high dimensional data},
  author={Liao, Zhenyu and Couillet, Romain},
  booktitle={International Conference on Machine Learning},
  pages={3063--3071},
  year={2018},
  organization={PMLR}
}

@article{louart2018random,
  title={A random matrix approach to neural networks},
  author={Louart, Cosme and Liao, Zhenyu and Couillet, Romain},
  journal={The Annals of Applied Probability},
  volume={28},
  number={2},
  pages={1190--1248},
  year={2018},
  publisher={JSTOR}
}

@article{misiakiewicz2024non,
  title={A non-asymptotic theory of kernel ridge regression: deterministic equivalents, test error, and gcv estimator},
  author={Misiakiewicz, Theodor and Saeed, Basil},
  journal={arXiv preprint arXiv:2403.08938},
  year={2024}
}

@article{hu2024asymptotics,
  title={Asymptotics of random feature regression beyond the linear scaling regime},
  author={Hu, Hong and Lu, Yue M and Misiakiewicz, Theodor},
  journal={arXiv preprint arXiv:2403.08160},
  year={2024}
}

@article{moniri2023theory,
  title={A theory of non-linear feature learning with one gradient step in two-layer neural networks},
  author={Moniri, Behrad and Lee, Donghwan and Hassani, Hamed and Dobriban, Edgar},
  journal={arXiv preprint arXiv:2310.07891},
  year={2023}
}

@article{fan2024kronecker,
  title={Kronecker-product random matrices and a matrix least squares problem},
  author={Fan, Zhou and Ma, Renyuan},
  journal={arXiv preprint arXiv:2406.00961},
  year={2024}
}

@article{benigni2021eigenvalue,
  title={Eigenvalue distribution of some nonlinear models of random matrices},
  author={Benigni, Lucas and P{\'e}ch{\'e}, Sandrine},
  journal={Electronic Journal of Probability},
  volume={26},
  pages={1--37},
  year={2021},
  publisher={The Institute of Mathematical Statistics and the Bernoulli Society}
}

@article{atanasov2024scaling,
  title={Scaling and renormalization in high-dimensional regression},
  author={Atanasov, Alexander and Zavatone-Veth, Jacob A and Pehlevan, Cengiz},
  journal={arXiv preprint arXiv:2405.00592},
  year={2024}
}

@article{fan2020spectra,
  title={Spectra of the conjugate kernel and neural tangent kernel for linear-width neural networks},
  author={Fan, Zhou and Wang, Zhichao},
  journal={Advances in neural information processing systems},
  volume={33},
  pages={7710--7721},
  year={2020}
}

@article{ba2022high,
  title={High-dimensional asymptotics of feature learning: How one gradient step improves the representation},
  author={Ba, Jimmy and Erdogdu, Murat A and Suzuki, Taiji and Wang, Zhichao and Wu, Denny and Yang, Greg},
  journal={Advances in Neural Information Processing Systems},
  volume={35},
  pages={37932--37946},
  year={2022}
}

@article{rubio2011spectral,
  title={Spectral convergence for a general class of random matrices},
  author={Rubio, Francisco and Mestre, Xavier},
  journal={Statistics \& probability letters},
  volume={81},
  number={5},
  pages={592--602},
  year={2011},
  publisher={Elsevier}
}

@inproceedings{lu2024context,
  title={In-context learning by linear attention: Exact asymptotics and experiments},
  author={Lu, Yue and Letey, Mary and Zavatone-Veth, Jacob A and Maiti, Anindita and Pehlevan, Cengiz},
  booktitle={NeurIPS 2024 Workshop on Mathematics of Modern Machine Learning},
  year={2024}
}

@article{lu2025equivalence,
  title={An equivalence principle for the spectrum of random inner-product kernel matrices with polynomial scalings},
  author={Lu, Yue M and Yau, Horng-Tzer},
  journal={The Annals of Applied Probability},
  volume={35},
  number={4},
  pages={2411--2470},
  year={2025},
  publisher={Institute of Mathematical Statistics}
}

@inproceedings{d2020double,
  title={Double trouble in double descent: Bias and variance (s) in the lazy regime},
  author={d’Ascoli, St{\'e}phane and Refinetti, Maria and Biroli, Giulio and Krzakala, Florent},
  booktitle={International Conference on Machine Learning},
  pages={2280--2290},
  year={2020},
  organization={PMLR}
}

@inproceedings{gerace2020generalisation,
  title={Generalisation error in learning with random features and the hidden manifold model},
  author={Gerace, Federica and Loureiro, Bruno and Krzakala, Florent and M{\'e}zard, Marc and Zdeborov{\'a}, Lenka},
  booktitle={International Conference on Machine Learning},
  pages={3452--3462},
  year={2020},
  organization={PMLR}
}

@article{rende2024distributional,
  title={A distributional simplicity bias in the learning dynamics of transformers},
  author={Rende, Riccardo and Gerace, Federica and Laio, Alessandro and Goldt, Sebastian},
  journal={Advances in Neural Information Processing Systems},
  volume={37},
  pages={96207--96228},
  year={2024}
}

@misc{devlin2018bert,
      title={BERT: Pre-training of Deep Bidirectional Transformers for Language Understanding}, 
      author={Jacob Devlin and Ming-Wei Chang and Kenton Lee and Kristina Toutanova},
      year={2019},
      eprint={1810.04805},
      archivePrefix={arXiv}
}

@misc{howard2018universal,
      title={Universal Language Model Fine-tuning for Text Classification}, 
      author={Jeremy Howard and Sebastian Ruder},
      year={2018},
      eprint={1801.06146},
      archivePrefix={arXiv},
      primaryClass={cs.CL}
}

@misc{radford2018improving,
  title={Improving language understanding by generative pre-training},
  author={Radford, Alec and Narasimhan, Karthik and others},
  year={2018}
}

@article{brown2020language,
  title={Language models are few-shot learners},
  author={Brown, Tom and Mann, Benjamin and Ryder, Nick and Subbiah, Melanie and Kaplan, Jared D and Dhariwal, Prafulla and Neelakantan, Arvind and Shyam, Pranav and Sastry, Girish and Askell, Amanda and others},
  journal={Advances in neural information processing systems},
  volume={33},
  pages={1877--1901},
  year={2020}
}

@misc{chatgpt,
      title={GPT-4 Technical Report}, 
      author={OpenAI},
      year={2024},
      eprint={2303.08774},
      archivePrefix={arXiv},
      primaryClass={cs.CL}
}

@article{bhattacharya2020single,
author = {Bhattacharya, Nicholas and Thomas, Neil and Rao, Roshan and Daupras, Justas and Koo, Peter and Baker, David and Song, Yun and Ovchinnikov, Sergey},
year = {2020},
title = {Single Layers of Attention Suffice to Predict Protein Contacts},
doi = {10.1101/2020.12.21.423882},
journal = {bioRxiv}
}

@article{viteritti2022transformer,
  title = {Transformer Variational Wave Functions for Frustrated Quantum Spin Systems},
  author = {Viteritti, Luciano Loris and Rende, Riccardo and Becca, Federico},
  journal = {Phys. Rev. Lett.},
  volume = {130},
  issue = {23},
  pages = {236401},
  numpages = {6},
  year = {2023},
  month = {Jun},
  publisher = {American Physical Society},
  doi = {10.1103/PhysRevLett.130.236401},
  url = {https://link.aps.org/doi/10.1103/PhysRevLett.130.236401}
}

@misc{rende2023simple,
      title={A simple linear algebra identity to optimize Large-Scale Neural Network Quantum States}, 
      author={Riccardo Rende and Luciano Loris Viteritti and Lorenzo Bardone and Federico Becca and Sebastian Goldt},
      year={2023},
      eprint={2310.05715},
      archivePrefix={arXiv},
      primaryClass={cond-mat.str-el}
}

@misc{viteritti2023transformer,
      title={Transformer Wave Function for the Shastry-Sutherland Model: emergence of a Spin-Liquid Phase}, 
      author={Luciano Loris Viteritti and Riccardo Rende and Alberto Parola and Sebastian Goldt and Federico Becca},
      year={2023},
      eprint={2311.16889},
      archivePrefix={arXiv},
      primaryClass={cond-mat.str-el}
}

@article{wu2020optimal,
  title={On the optimal weighted $\backslash$$ell\_2 $ regularization in overparameterized linear regression},
  author={Wu, Denny and Xu, Ji},
  journal={Advances in Neural Information Processing Systems},
  volume={33},
  pages={10112--10123},
  year={2020}
}

@article{HMRT22,
  title={Surprises in high-dimensional ridgeless least squares interpolation},
  author={Hastie, Trevor and Montanari, Andrea and Rosset, Saharon and Tibshirani, Ryan J},
  journal={Annals of statistics},
  volume={50},
  number={2},
  pages={949},
  year={2022}
}

@article{DLM24,
  title={Dimension-free deterministic equivalents and scaling laws for random feature regression},
  author={Defilippis, Leonardo and Loureiro, Bruno and Misiakiewicz, Theodor},
  journal={Advances in Neural Information Processing Systems},
  volume={37},
  pages={104630--104693},
  year={2024}
}
\clearpage
\clearpage
\bibliographystyle{abbrvnat}

\newpage
\appendix
\onecolumn
\section{Preliminaries}
\label{appendix_preliminaries}

Recall the definition of $\hat{A}$: 
\begin{equation}
    \hat{A}: = [\hat{a}_1, \dots, \hat{a}_d] \in \RR^{d \times d}, 
\end{equation}
where $\hat{a}_k$ solves \Cref{eq:coordinate_wise_regression}. We can first re-write as matrix optimization problem: 
\begin{equation}
    \hat{A} = \underset{\substack{A\in\mathbb{R}^{d\times d}\\ {\rm diag}(A)=0}}{\argmin}\frac{1}{n}\|X-XA^{\top}||^{2}_{F}+\lambda||A||_{F}^{2}.
    \label{eq:joint_matrix_opt}
\end{equation}

\subsection{Proof of \ref{lemma:explicit_expression_A}}
We solve this problem by imposing the KKT conditions in \cref{eq:joint_matrix_opt}. Let $g_l(A) = A_{l,l}$ and $F(A)=  \frac{1}{n} \| X - X A^T\|_{2}^2 + \lambda \| A\|_{F}^2$. Then, the Lagrangian is given by: 
\begin{align}
    \mathcal{L}(A) & = F(A) - \sum_{l=1}^{d} \lambda_l g_{l}(A) = F(A) - \mathrm{Tr}(\Lambda A )\\
    & = \dfrac{1}{n}\mathrm{Tr} \left ( (I_{d} - A^T) X^T X (I_{d} - A) \right ) + \lambda \mathrm{Tr} (A^T A) - \mathrm{Tr} (\Lambda A)
\end{align}
where $\Lambda = \diag(\lambda_1, \dots, \lambda_d)$. Then, differentiating with respect to $A$, we have: 
\begin{equation}
    \frac{\partial}{\partial A} \mathcal{L} (A) =  -\dfrac{2}{n} X^T X + \dfrac{2}{n} X^T X A + 2\lambda A + \Lambda,
\end{equation}
and imposing the stationarity condition of the Lagrangian we get: 
\begin{equation}
    \frac{\partial \mathcal{L}}{\partial A}   = -\frac{2}{n} X^TX  + \frac{2}{n}(X^TX + \lambda I_{d}) A + \Lambda = 0.
\end{equation}
Denote $Q({\lambda}): = (\hat{\Sigma} + \lambda I_{d})^{-1}$. Then, solving for $A$ gives: 
\begin{equation}
    A = (\frac{1}{n}X^TX + \lambda I_{d})^{-1}( \frac{1}{n}X^TX + \frac{1}{2}\Lambda) = Q(\lambda) (\hat{\Sigma} + \dfrac{1}{2} \Lambda).
\end{equation}
In order to obtain $\Lambda$, we impose the restriction of the optimization problem: 
\begin{align}
    &\mathrm{diag} ( \hat{A} ) =  \mathrm{diag} \left ( (\hat{\Sigma} + \lambda I_{d}))^{-1}(\hat{\Sigma} + \frac{1}{2}\Lambda)\right ) = 0\\
     & \iff  \lambda_k Q({\lambda})_{k,k}  = -2Q({\lambda })_{k,:}^T \hat{\Sigma}_{k,:} \\
     & \iff  \lambda_k = \dfrac{-2}{Q({\lambda})_{k,k}}Q({\lambda })_{k,:}^T \hat{\Sigma}_{k,:},
\end{align}
where we used the fact that $\Lambda$ is diagonal. Then, by defining $D(\lambda) \in \RR^{d \times d}$ the diagonal matrix with entries: 
\begin{equation}
    D(\lambda)_{k,k} = \dfrac{\diag(Q(\lambda) \hat{\Sigma})_{k,k}}{Q({\lambda})_{k,k}},
\end{equation}
we can write: 
\begin{equation}
    \hat{A} = Q(\lambda)(\hat{\Sigma} - D(\lambda)). 
    \label{eq:solution_A}
\end{equation}
By adding and subtracting the $\frac{\lambda}{n}$ in \eqref{eq:solution_A}, we also get: 
\begin{equation}
    \hat{A}  = I_d  - Q(\lambda)(D(\lambda)  + \lambda I_d).
\end{equation}
Note that 
\begin{align}
    D(\lambda) + \lambda I_d & = \dfrac{\diag(Q(\lambda) \hat{\Sigma})}{\diag(Q({\lambda}))}  +\lambda I_d\\
    & = \dfrac{1}{\diag(Q(\lambda))} \left ( \diag(Q(\lambda) \hat{\Sigma}) + \lambda \diag(Q(\lambda)) \right ) \\
    & = \dfrac{1}{\diag(Q(\lambda))}. 
\end{align}
Hence, we finally write: 
\begin{equation}
    \hat{A} = I_d - Q(\lambda) [\diag(Q(\lambda))]^{-1}. 
\end{equation}

\subsection{Concentration}

\begin{lemma}
Consider a matrix $X \in \RR^{n \times d}$, where each row is an independent sample from $\mathcal{N}(0,\Sigma)$, and $\hat{A}$ the SSR matrix. As $n, d \to \infty$, with high probability, we have: 
\[
\| \hat{A} - (I_{d} -  Q(\lambda) \bar{D}\|_\mathrm{op} \overset{d \to \infty}{\to} 0,
\]
where $\bar{D} \in \RR^{d \times d}$ is a diagonal matrix with entries:
\[
\bar{D}_{k,k} : =  \dfrac{\lambda}{[(I_{d} + \tilde{m}(\lambda) \Sigma )^{-1}]_{k,k}},
\]
where $\tilde{m}$ solves the self-consistent equation 
\begin{equation}
    \tilde{m} = \left ( \lambda + \dfrac{1}{n}\mathrm{Tr}\left ( \Sigma (\tilde{m}(\lambda) \Sigma + I_d )^{-1}\right )\right )^{-1}.
\end{equation}
\label[lemma]{lemma:approximation_Lambda_D}
\end{lemma}

\begin{proof}
Let $X = Z \Sigma^{\frac{1}{2}} \in \RR^{n \times d}$, where $Z$ is a matrix with standard gaussian, independent entries, and let $Q(\lambda) := (\hat{\Sigma} + \lambda I_{d})^{-1}$, for $\lambda >0$. By \Cref{lemma:explicit_expression_A}, we have: 
\begin{equation}
    \hat{A} = I_{d} - Q(\lambda)  [\diag(Q(\lambda))]^{-1}.
\end{equation}
We will to prove that, for $d$ big enough, we can replace $[\diag(Q(\lambda))]^{-1}$ by $[\diag(\bar{Q}(\lambda))]^{-1}$, where
\begin{equation}
    \bar{Q}(\lambda) := -\frac{1}{\lambda}(\tilde{m}(\lambda) \Sigma + I_{d})^{-1},
\end{equation}
is the deterministic equivalent of $\bar{Q}(\lambda)$ (Theorem 2.6, \cite{CL22}). Let $\bar{A}:= \hat{A} = I_{d} - Q(\lambda)  [\diag(\bar{Q}(\lambda))]^{-1}$. We have: 
\begin{align}
    \| \hat{A} - \bar{A}\|_\mathrm{op} &= \|Q(\lambda) ([\diag(Q(\lambda))]^{-1} -  [\diag(\bar{Q}(\lambda))]^{-1}) \|_\mathrm{op}.
\end{align}
By using the fact that 
\begin{equation}
    \| Q(\lambda) \|_\mathrm{op} \leq \dfrac{1}{\lambda},
\end{equation}
we get: 
\begin{align}
    \| \hat{A} - \bar{A}\|_\mathrm{op} & \leq \dfrac{1}{\lambda} \max_{k \in [d]} \left | \dfrac{1}{Q(\lambda)_{k,k}} - \dfrac{1}{\bar{Q}(\lambda)_{k,k}}\right |.
\end{align}
Now using the identity $A^{-1} - B^{-1} = A^{-1}(B - A) B^{-1}$: 
\begin{align}
    \| \hat{A} - \bar{A}\|_\mathrm{op} & \leq C \max_{k \in [d]} \left | Q(\lambda)_{k,k} - \bar{Q}(\lambda)_{k,k} \right | \\ 
    & \leq C \max_{k \in [d]} \left | e_k^TQ(\lambda) e_k - e_k^T \bar{Q}(\lambda)e_k \right |,
\end{align}
where we used that $\|\hat{\Sigma}\|_\mathrm{op} \leq C$ with high probability (\cite{V10}, Theorem 5.44). From here, we conclude by using the fact that $\bar{Q}(\lambda)$ is the deterministic equivalent of $Q(\lambda)$, so $a^T Q(\lambda) b \to a^T \bar{Q}(\lambda) b$, for $a, b \in \mathbb{S}^{d-1}$ and a union-bound on each $k \in [d]$. 
\end{proof}
\section{The Spectrum of $\hat{A}$}
\label{section:appendix_spectrum}

With the results in Appendix \ref{appendix_preliminaries}, we 
 we can proceed to study the spectrum of the SSR matrix $\hat{A}$, given by: 
 \begin{equation}
     \hat{A}: = I_{d} - Q(\lambda) \Lambda, \quad \Lambda = [\diag(Q(\lambda))]^{-1}
 \end{equation}
 The first challenge we note is the second term: We have the product of $Q(\lambda)$ with a diagonal matrix $\Lambda$, which are both random but not independent. We will overcome this complication by leveraging the fact that the diagonal of the resolvent is, under mild conditions, very concentrated. This is where we will first apply \Cref{lemma:approximation_Lambda_D}. With this, we will be able to replace the diagonal matrix $[\diag(Q(\lambda))]^{-1}$ by a deterministic matrix $\bar{D}$, and study the matrix $I_{d} - Q(\lambda) \bar{D}$ instead of $\hat{A}$.  
 
\subsection{The non-isotropic case - Proof of \Cref{thm:det_equivalent_A_hat}}
\label{appendix_non_isotropic}
In what follows, we directly study the spectrum of $\hat A$. Since the diagonal entries of $Q$ are nonnegative, we only need to study the spectrum of the matrix
\begin{equation}
    \bar{D}^{\frac{1}{2}} Q(\lambda) \bar{D}^{\frac{1}{2}},
\end{equation}
where, by \Cref{lemma:approximation_Lambda_D}
\begin{equation}\label{eq:Lambda_def}
    \bar{D}_{k,k} : =  \dfrac{\lambda}{[(I_{d} + \tilde{m}(\lambda) \Sigma )^{-1}]_{k,k}}
\end{equation}
To make notation simpler, denote $\tilde{D} = \bar{D}^{\frac{1}{2}}$. For any $z \in \C^+$, we aim to obtain a deterministic equivalent of the resolvent
\begin{equation}
    G(z) = \left ( \tilde{D} Q(\lambda) \tilde{D} - z I_{d}\right )^{-1}.
\end{equation}
To that end, we use the linearization trick. Note that: 
\begin{equation}
    \left ( \tilde{D} Q(\lambda) \tilde{D} - z I_{d}\right ) G(z)  = I_{d}. 
\end{equation}
Now define $G_1 = Q(\lambda) \tilde{D} G(z)$, so $\tilde{D} G(z) + (-\hat{\Sigma} + \lambda I_{d})G_1 = 0 $. We get: 
\begin{equation}
      \tilde{D} G_1 - z G(z) = I_{d}.
\end{equation}
Putting these equations together, we get the following linear system of equations: 
\begin{equation}
    \begin{pmatrix}
        -zI_{d} & \tilde{D} \\
        \tilde{D}& - \hat{\Sigma} -\lambda I_{d} 
    \end{pmatrix} \begin{pmatrix}
        G(z) \\
        G_1(z)
    \end{pmatrix} = \begin{pmatrix}
        I_{d} \\ 0
    \end{pmatrix}.
\end{equation}
Let $H$ be the block matrix: 
\begin{equation}
    H = \begin{bmatrix}
        -z I & \tilde{D}\\
        \tilde{D} & -\hat \Sigma - \lambda I
    \end{bmatrix}.
\end{equation}
Then: 
\begin{align}
    G(z) = \begin{pmatrix}
        I_{d} & 0 
    \end{pmatrix} 
    \begin{pmatrix}
        G(z) \\
        G_1
    \end{pmatrix} = \begin{pmatrix}
        I_{d} & 0 
    \end{pmatrix} H^{-1} \begin{pmatrix}
        I_{d} \\ 0
    \end{pmatrix} = [H^{-1}]_{1,1}. 
\end{align}
We can also compute $H^{-1}$ via matrix block-inversion: 
\begin{align}
    H^{-1} &= \begin{bmatrix}
        G(z) & G(z) \tilde{D} Q\\
        Q\tilde{D} G(z) & (\frac{\tilde{D}^2}{z} - \hat \Sigma - \lambda I)^{-1}
    \end{bmatrix}\\
    &= \begin{bmatrix}
        G(z) & G(z) \tilde{D} Q\\
        Q \tilde{D} G(z) & -Q + Q \tilde{D}G(z)\tilde{D} Q
    \end{bmatrix}.
\end{align}
Note that
\begin{equation}
    H = \underbrace{\begin{pmatrix}
        -zI_{d} & \tilde{D} \\
        \tilde{D}  & -\lambda I_{d} 
    \end{pmatrix}}_{H_0:=} + \underbrace{\begin{pmatrix}
        0 & 0 \\
        0 & -\hat{\Sigma}
    \end{pmatrix}}_{\text{random part}}.
\end{equation}
The idea is to get an approximation of $\mathbb{E}[H^{-1}]$ and use it as a deterministic equivalent for $H$. For this, we first use the fact that $HH^{-1} = I_{2d}$:
\begin{align}
    H^{-1}H = I_{2d} \iff H^{-1}H_0 +H \begin{pmatrix}
        0 & 0 \\
        0 & -\hat{\Sigma} 
    \end{pmatrix}  = I_{2d}.
    \label{eq:H_H_inv_1}
\end{align}
Now, since $\hat{\Sigma} = \dfrac{1}{n} \sum_{i=1} x_i x_i^T$, we can define $\tilde{x}_i = (0_{d}, x_i)$, where $0_d$ denotes the $d$-zeros vector, and we will get: 
\begin{equation}
    \begin{pmatrix}
        0 & 0 \\
        0 & -\hat{\Sigma} 
    \end{pmatrix} = -\dfrac{1}{n}\sum_{i=1} \tilde{x}_i \tilde{x}_i^T. 
\end{equation}
Then \Cref{eq:H_H_inv_1} becomes: 
\begin{equation}
    H^{-1} H_0 -\dfrac{1}{n}\sum_{i=1}^{n} H^{-1} \tilde{x}_i \tilde{x}_i^T = I_{2d},
\end{equation}
and taking expectation, since $H_0$ is deterministic:  
\begin{equation}
    \EE[H^{-1} ] H_0 -\dfrac{1}{n}\sum_{i=1}^{n} \EE \left [ H^{-1} \tilde{x}_i \tilde{x}_i^T \right ] = I_{2d}
    \label{eq:H_H_inv_2}
\end{equation}
Now, we define the leave-one-out version of $H$. For each $i \in [n]$, let
\begin{equation}
    H_{\setminus i} = H_0 - \dfrac{1}{n} \sum_{j \not = i} \tilde{x}_j \tilde{x}_j^T.
\end{equation}
Then, by the Sherman-Morrison formula we have that: 
\begin{align}
    H^{-1} & =  H_{\setminus i}^{-1}-\dfrac{H_{\setminus i}^{-1} \tilde{x}_i \tilde{x}_i^T H_{\setminus i}^{-1}}{1 - \frac{1}{n}\tilde{x}^T H_{\setminus i}^{-1} \tilde{x}_i}. 
\end{align}
Then, we can write the second term in \Cref{eq:H_H_inv_2} (without the expectation) as:
\begin{align}
    H^{-1} \tilde{x}_i \tilde{x}_i^T & = H_{\setminus i}^{-1}\tilde{x}_i \tilde{x}_i^T + \dfrac{H_{\setminus i}^{-1} \tilde{x}_i \tilde{x}_i^T H_{\setminus i}^{-1}}{1 - \frac{1}{n}\tilde{x}^T H_{\setminus i}^{-1} \tilde{x}_i}\tilde{x}_i \tilde{x}_i^T  \\
    & = H_{\setminus i}^{-1} \tilde{x}_i \tilde{x}_i^T \dfrac{1}{1 -\frac{1}{n}\tilde{x}_i^T H_{\setminus i}^{-1} \tilde{x}_i }. 
    \label{eq:leave_one_out_identity}
\end{align}
Note that by concentration of quadratic forms (\cite{CL22}, Lemma 2.1), with very high probability: 
\begin{align}
     \frac{1}{n}\tilde{x}_i^T H_{\setminus i}^{-1} \tilde{x}_i  & =\dfrac{1}{n}  \mathrm{Tr} \left ( \Sigma [H_{\setminus i}^{-1}]_{2,2}\right )  + o_d(1) \\
     & = \dfrac{1}{n}\mathrm{Tr} \left ( \Sigma  \EE[H^{-1}_{\setminus i}]_{2,2}\right ) + o_d(1),
     \label{eq:slef_consistent_chi_concentration_1}
\end{align}
where the last line follows from the fact that $\mathbb{E}[H^{-1}_{\setminus i}$ is a deterministic equivalent for $H^{-1}_{\setminus i}$, and $[\cdot]_{2,2}$ denotes the second block of the matirx (coordinates $d+1 $ to $2d$). Let $\chi = \dfrac{1}{n}\mathrm{Tr} \left ( \Sigma  \EE[H^{-1}_{\setminus i}]\right )$. Then applying \Cref{eq:leave_one_out_identity} in \Cref{eq:H_H_inv_2}: 
\begin{align}
    & \EE[H^{-1}] H_0 - \dfrac{1}{n(1 -\chi)}\sum _{i=1}^{n} \EE\left [ H^{-1}_{\setminus i} \tilde{x}_i \tilde{x}_i^T\right ] = I_{2d} \\
    \iff & \EE[H^{-1}] H_0 - \dfrac{1}{n(1- \chi)} \sum_{i=1}^{n} \EE\left [ H^{-1}_{\setminus i} \right ] \begin{pmatrix}
        0 & 0 \\
        0 & \Sigma
    \end{pmatrix}= I_{2d}, \\
    \iff & \EE[H^{-1}] H_0 + \EE\left [ H^{-1}_{\setminus i} \right ] \begin{pmatrix}
        0 & 0 \\
        0 & \dfrac{1}{(1- \chi)}\Sigma
    \end{pmatrix}= I_{2d}.
\end{align}
where in the second to last line we used the independence of $H_{\setminus i}$ and $\tilde{x}_i$. Now, using that $\EE[H_{\setminus i}^{-1}] = \EE[H^{-1}] + o_{n}(1)$, we get: 
\begin{equation}
    \EE[H^{-1}] \left ( H_0 + \begin{pmatrix}
        0 & 0 \\
        0 & -\dfrac{1}{(1- \chi)}\Sigma
    \end{pmatrix} \right ) = I_{2d} + o_{n}(1). 
\end{equation}
Therefore: 
\begin{equation}
    \EE[H^{-1}] \sim  \begin{pmatrix}
        -zI & \tilde{D}\\
        \tilde{D} & -\dfrac{1}{(1- \chi)}\Sigma-\lambda I_{d}
    \end{pmatrix}^{-1}. 
\end{equation}
At last, recalling \Cref{eq:slef_consistent_chi_concentration_1}, by the matrix inversion lemma we have: 
\begin{equation}
    \EE[H^{-1}_{\setminus i}]_{2,2} \sim \left [ \begin{pmatrix}
        -zI & \tilde{D} \\
        \tilde{D} & -\dfrac{1}{(1- \chi)}\Sigma -\lambda I_{d}
    \end{pmatrix}^{-1}\right ]_{2,2} = \left ( \dfrac{\tilde{D}^2}{z}-\frac{1}{1-\chi} \Sigma -\lambda I_{d}\right )^{-1},
\end{equation}
and then, asymptotically $\chi$ solves::  
\begin{equation}
    \chi = \dfrac{1}{n}\mathrm{Tr} \left ( \Sigma  \left ( \dfrac{\tilde{D}^2}{z}-\frac{1}{1-\chi} \Sigma -\lambda I_{d}\right )^{-1}\right ). 
\end{equation}
In conclusion, we've find that:
\begin{equation}
    H^{-1} \sim \begin{pmatrix}
        -zI & \tilde{D} \\
        \tilde{D} & -\dfrac{1}{(1- \chi)}\Sigma -\lambda I_{d}
    \end{pmatrix}^{-1},
\end{equation}
and then, since $G(z)$ is the upper left block of $H^{-1}$, we have:
\begin{equation}
    G(z) \sim \left ( \tilde{D} \left ( \dfrac{1}{(1- \chi)}\Sigma +\lambda I_{d} \right)^{-1} \tilde{D} -zI_{d}\right )^{-1},
\end{equation}
where $\chi$ solves: 
\begin{equation}
    \chi = \dfrac{1}{n}\mathrm{Tr} \left ( \Sigma  \left ( \dfrac{\tilde{D}^2}{z}-\frac{1}{1-\chi} \Sigma -\lambda I_{d}\right )^{-1}\right ). 
\end{equation}
We conclude \Cref{thm:det_equivalent_A_hat} by recalling that $\tilde{D} = \bar{D}^{\frac{1}{2}}$.

\subsection{Diagonal Population Covariance in the Ridgeless Limit - Proof of \Cref{corollary:universality}}
\label{section_diagonal_ridgeless_limit}

In what follows, we work out a special case corresponding to $\Sigma$ being a diagonal matrix and the ridge parameter $\lambda \to 0^+$.

First, under assumptions on the moments and the data distribution being centered, the deterministic equivalent of the resolvent $Q(\lambda) = (\hat{\Sigma} + \lambda I_{d})^{-1}$ is given by (\cite{CL22}, Theorem 2.6): 
\begin{equation}
    \bar{Q}(\lambda) = \dfrac{(\tilde{m}(\lambda) \Sigma + I_{d})^{-1}}{\lambda} = (\lambda\tilde{m}(\lambda) \Sigma + \lambda I_{d})^{-1},
\end{equation}
where $\tilde{m}(\lambda)$ solves: 
\begin{equation}
    \tilde{m}(\lambda) = \left ( \lambda + \dfrac{1}{n} \mathrm{Tr} \left ( \Sigma (\tilde{m}(\lambda)\Sigma + I_{d})^{-1}\right )\right )^{-1}.
\end{equation}
Let $\dfrac{1}{1 + \nu}:= \lambda \tilde{m}(\lambda)$. Then: 
\begin{equation}
    \bar{Q}(\lambda) = (\dfrac{1}{\nu + 1} \Sigma + \lambda I_{d})^{-1},
\end{equation}
and $\nu$ satisfies: 
\begin{equation}
    \nu = \dfrac{1}{n} \mathrm{Tr} \left (\Sigma \left (\dfrac{1}{\nu +1} \Sigma + \lambda I_{d} \right )^{-1} \right ).
    \label{self_consistent_nu}
\end{equation}
Taking $\lambda >0$, when $\alpha >1$ we get: 
\begin{equation}
    \nu = \dfrac{1}{n} \mathrm{Tr} \left (\Sigma \left (\dfrac{1}{\nu +1} \Sigma\right )^{-1} \right ) = \dfrac{d(\nu +1)}{n} = \dfrac{\nu + 1}{\alpha},
\end{equation}
so the solution is $\nu = \frac{1}{\alpha - 1}$. Then, in this case the matrix $\Lambda$ concentrates into: 
\begin{equation}
    \Lambda = [\mathrm{diag} \left (Q(\lambda) \right ) ]^{-1} \sim \bar{D} = \dfrac{1}{\nu + 1} \mathrm{diag}(\Sigma^{-1})^{-1},
\end{equation}
and since $\Sigma$ is diagonal, we conclude: 
\begin{equation}
    \bar{D} = \dfrac{\alpha -1}{\alpha} \Sigma. 
\end{equation}
Recall that the deterministic equivalent of Theorem \Cref{thm:det_equivalent_A_hat} is for $\bar{D}^{\frac{1}{2}} Q(\lambda)\bar{D}^{\frac{1}{2}}$, and in this case: 
\begin{equation}
    \bar{D} ^{\frac{1}{2}} = \sqrt{\dfrac{\alpha -1}{\alpha}} \Sigma^{\frac{1}{2}}. 
\end{equation}
Then, by \Cref{thm:det_equivalent_A_hat}, for $z \in \mathbb{C}$, the resolvent has the following deterministic equivalent: 
\begin{align}
    G(z) & \sim \left ( \tilde{D} \left ( \dfrac{1}{(1- \chi)}\Sigma\right)^{-1} \tilde{D} -zI_{d}\right )^{-1} \\
    & = \left ( \dfrac{\alpha - 1}{\alpha} \Sigma^{\frac{1}{2}}\left ( \dfrac{1}{(1- \chi)}\Sigma \right)^{-1} \Sigma^{\frac{1}{2}} -zI_{d}\right )^{-1} \\
    & = \left ( \dfrac{(\alpha - 1)(1 - \chi)}{\alpha} I_{d} - z I_{d} \right )^{-1},
    \label{eq:det_equivalent_resolvent_diagonal}
\end{align}
where $\chi$ solves: 
\begin{align}
   \chi^2 + (z-1)\chi + \dfrac{z}{\alpha - 1} = 0. 
   \label{eq:self_consistent_equation_diagonal}
\end{align}
We conclude the independence of the spectrum from the elements of $\Sigma$ by noting that \Cref{eq:det_equivalent_resolvent_diagonal} and \Cref{eq:self_consistent_equation_diagonal} are independent of $\Sigma$.

\section{Asymptotic Limits for the Training and Generalization Errors}
\label{appendix_det_equivaletns}

\subsection{Proof of Lemma \ref{lemma:approximation_error}}
\label{appendix_app_error}

In this section, we prove \Cref{lemma:approximation_error}. Let $\Sigma \in \RR^{d \times d}$, with $\Sigma \succ 0$. Then we can write: 
    Indeed, we can write: 
    \begin{align}
        L(A) & = \EE_{x} \left [ (x_i - A x_i)^T (x_i - A x_i)\right ] \\
        & = \mathrm{Tr}(\Sigma) - 2 \mathrm{Tr}(A \Sigma)  +  \mathrm{Tr} (A^T A \Sigma)
    \end{align}
    Then, the Lagrangian of this problem is: 
    \begin{equation}
        \mathcal{L}(A) = \mathrm{Tr}(\Sigma) - 2 \mathrm{Tr}(A^T \Sigma) +  \mathrm{Tr} (A^T A \Sigma) - \mathrm{Tr}(\Lambda A),
    \end{equation}
    where $\Lambda$ is a diagonal matrix. Hence, by imposing the stationarity of the Lagrangian, and using the assumption that $\Sigma$ is invertible, we get: 
    \begin{equation}
        - 2 \Sigma + 2\Sigma A - \Lambda = 0 \implies A = \Sigma^{-1} \left ( \frac{1}{2}\Lambda + \Sigma \right ). 
    \end{equation}
    Imposing the condition $\mathrm{diag}(A) = 0$, we get: 
    \begin{equation}
        A_{\mathrm{App}} = \Sigma^{-1} ( \Sigma - \mathrm{diag}(\Sigma^{-1})) = I_d - \Sigma^{-1} [\mathrm{diag}(\Sigma^{-1})]^{-1}.
    \end{equation}
    By replacing $A_{\mathrm{App}}$ in the generalization error, we get that the Approximation error is then given by: 
    \[
    L(A_{\mathrm{App}}) = \mathrm{Tr} [ \Sigma^{-1} [\mathrm{diag}(\Sigma^{-1})]^{-2 }] . 
    \]

\subsection{Proof of \Cref{theorem:det_equivalent_gen_error}}

In this section, we prove the asymptotic limit for $L(\hat{A})$. From \Cref{eq:gen_error_expression} and \cref{lemma:explicit_expression_A}, we have: 
\begin{align}
    L(\hat{A}) & = \dfrac{1}{d} \mathrm{Tr} \left ( (I_{d} - \hat{A}^T) \Sigma (I_{d} + \hat{A} )\right )  \\
    & = \dfrac{1}{d} \mathrm{Tr} \left (\Lambda Q(\lambda) \Sigma Q(\lambda) \Lambda \right ),
\end{align}
where we recall that $Q(\lambda) = (\hat{\Sigma} + \lambda I_{d})^{-1}$ and $\Lambda = [\mathrm{diag}(Q(\lambda))]^{-1}$. As we saw in Appendix \ref{appendix_preliminaries}, we have that
\begin{equation}
    \hat{A} \approx I_{d} - Q(\lambda) \bar{D},
\end{equation}
where $\bar{D} = \lambda [\mathrm{diag}(Q(\kappa(\lambda) ))]^{-1}$, and $\tilde{m}(\lambda)$ is a solution to the self-consistent equation: 
\begin{equation}
    \tilde{m}(\lambda) = \left ( \lambda + \dfrac{1}{n}\mathrm{Tr}\left ( \Sigma (\tilde{m}(\lambda) \Sigma + I_d )^{-1}\right )\right )^{-1}.
\end{equation}
Our first step in proving an asymptotic limit for $L(\hat{A})$ is to prove that we can approximate this quantity by $L(I_{d} - Q(\lambda)\bar{D})$. This is what we do in the following lemma. 

\begin{lemma}
    Let $X = Z \Sigma \in \RR^{n \times d}$, with $\| \Sigma \|_\mathrm{op}$ bounded, and $Z$ having independent entries, with mean $0$, variance $1$, and $4 + \varepsilon$ bounded moments. As $d$ grows to infinity, the normalized risk converges to: 
    \begin{equation}
        L(\hat{A}) \to  \dfrac{1}{d}\mathrm{Tr} [\bar{D}Q(\lambda) \Sigma Q(\lambda )\bar{D} ] .
    \end{equation}
    \label[lemma]{lemma:concentration_gen_error_1}
\end{lemma}
\begin{proof}
    Recall that $\hat{A} = I_{d} - Q(\lambda )\Lambda $, and as we just saw: 
    \begin{equation}
        L(\hat{A}) = \mathrm{Tr} [D Q(\lambda) \Sigma Q(\lambda) D] = \mathrm{Tr} [D^2 Q(\lambda) \Sigma Q(\lambda)].
    \end{equation}
    Then: 
    \begin{align}
        L(\hat{A}) = \mathrm{Tr} [\bar{D}^2 Q(\lambda) \Sigma Q(\lambda)] + \mathrm{Tr} [\Delta^2 Q(\lambda) \Sigma Q(\lambda)],
    \end{align}
    where $\Delta = \Lambda - \bar{D}$. 
    Therefore, we need to show that with high probability, $\mathrm{Tr} [\Delta^2 Q(\lambda) \Sigma Q(\lambda)] \to 0$, as $d \to \infty$. We have: 
    \begin{align}
        \dfrac{1}{d}\mathrm{Tr} [\Delta^2 Q(\lambda) \Sigma Q(\lambda)] & \leq \frac{1}{d} d \| \Delta\|_\mathrm{op}^2 \|Q(\lambda) \Sigma Q (\lambda) \|_\mathrm{op}  \\
        & \leq  C \| \Delta \|_\mathrm{op}^2,
    \end{align}
    where we used that $\| \Sigma \|_\mathrm{op} $ is bounded by assumption, $\|Q(\lambda) \|_\mathrm{op} \leq \frac{1}{\lambda}$, and $\| \Delta \|_\mathrm{op} \to 0$, as $d \to \infty$ by \Cref{lemma:approximation_Lambda_D}. 
\end{proof}

Before proving the asymptotic limit, we need one more Lemma. 
\begin{lemma}[Proposition 1 in \cite{B24}] 
Let $X = Z \Sigma \in \RR^{n \times d}$, with $\| \Sigma \|_\mathrm{op}$ bounded, and $Z$ having independent entries, with mean $0$, variance $1$. Assume that $\Sigma$ has a limiting spectral density and that $A, B$ have bounded operator norm. Then: 
\begin{align*}
    \mathrm{Tr}\left ( A (\hat{\Sigma} + \lambda I_{d})^{-1} \right ) &\sim  \dfrac{\kappa(\lambda)}{\lambda} \mathrm{Tr} \left ( AQ(\kappa(\lambda)) \right )
\end{align*}
and 
\begin{align*}
\mathrm{Tr}\left (  A (\hat{\Sigma} + \lambda I_{d})^{-1} B (\hat{\Sigma} + \lambda I_{d})^{-1}\right ) & \sim  \frac{\kappa(\lambda )^2}{\lambda^2  } \left (  \mathrm{Tr} \left ( A (\Sigma + \kappa(\lambda ) I_d)^{-1} B(\Sigma + \kappa(\lambda ) I_{d})^{-1} \right )\right . \\
& \left .+\frac{1}{n - \mathrm{Tr} \left ( \Sigma^2 (\Sigma + \kappa(\lambda ) I_{d})^{-2}\right )} \mathrm{Tr} \left ( A(\Sigma + \kappa(\lambda )I_{d})^{-2}\Sigma \right ) \mathrm{Tr} \left ( B(\Sigma + \kappa(\lambda )I_{d})^{-2} \Sigma \right )\right ),
\end{align*}
where $\frac{1}{\kappa(\lambda)}$ solves the self-consistent equation: 
\begin{equation*}
     \kappa(\lambda) = \lambda + \dfrac{\kappa(\lambda)}{n} \mathrm{Tr} \left ( \Sigma ( \Sigma + \kappa(\lambda)I_{d})^{-1}\right )
\end{equation*}
\label[lemma]{lemma:det_equivalent_quadratic}
\end{lemma}

We can now proceed with the proof of \Cref{theorem:det_equivalent_gen_error}. First, by \Cref{lemma:concentration_gen_error_1}, we have: 
\begin{equation}
    L(\hat{A} ) \sim  \dfrac{1}{d} \mathrm{Tr} \left (\bar{D}^2 Q(\lambda) \Sigma Q(\lambda) \right ). 
\end{equation}
Then, by applying \Cref{lemma:det_equivalent_quadratic} we get: 

\begin{align}
    L(\hat{A}) & \sim  \frac{\kappa(\lambda)^2}{\lambda^2 } \left ( \mathrm{Tr} \left ( \bar{D}^2 (\Sigma + \kappa(\lambda) I_{d})^{-1} \Sigma (\Sigma + \kappa(\lambda) I_{d})^{-1}\right )\right . \\
    & \left .+ \dfrac{1}{n - \mathrm{Tr} \left ( \Sigma^2 (\Sigma + \kappa(\lambda) I_{d} )^{-2}\right)} \mathrm{Tr} \left (\bar{D}^2(\Sigma + \kappa(\lambda))^{-2} \Sigma \right ) \mathrm{Tr} \left ( \Sigma (\Sigma + \kappa(\lambda) I_{d})^{-2} \Sigma \right) \right ). \label{eq:det_equivalent_indermediate_step_1}
\end{align}
Let $L_1:=\mathrm{Tr} \left ( \bar{D}^2 (\Sigma + \kappa(\lambda) I_{d})^{-1} \Sigma (\Sigma + \kappa(\lambda) I_{d})^{-1}\right )$. Since $\Sigma$ and $(\Sigma + \kappa(\lambda) I_{d})^{-1}$ commute, we can re-write \Cref{eq:det_equivalent_indermediate_step_1} as: 
\begin{align}
    L(\hat{A}) & \sim  \frac{\kappa(\lambda)^2}{\lambda^2 } L_1   \left ( 1 + \dfrac{1}{n - \mathrm{Tr} \left ( \Sigma^2 (\Sigma + \kappa(\lambda) I_{d} )^{-2}\right)} L_{1}\mathrm{Tr} \left ( \Sigma (\Sigma + \kappa(\lambda) I_{d})^{-2} \Sigma \right) \right )
\end{align}
By re-writing $\bar{D} = \dfrac{\kappa(\lambda)}{\lambda} [\mathrm{diag}(Q(\kappa(\lambda)))]^{-1} = \dfrac{\kappa(\lambda)}{\lambda} \bar{D}_2$, we get: 
\begin{align}
    L(\hat{A}) & \sim L_1 \left ( 1 + \dfrac{\mathrm{df}_2^{\Sigma}(\kappa(\lambda))}{n - \mathrm{df}_2^{\Sigma}(\kappa(\lambda))} \right ).
\end{align}

\section{Comparison with PCA in the population limit}
\label{appendix_pca}

In this section, we consider the classical statistical limit $n \to \infty$, while $d = \Theta_{n}(1)$. Note that in this case, $\hat{\Sigma} \to \Sigma$, and therefore all the quantities in the population limit, and in particular the resolvent, will contain the population covariance $\Sigma$. To avoid confusion, we will denote
\[
Q_{\mathrm{App}}(\lambda) := (\Sigma + \lambda I_{d})^{-1}.
\] 

\subsection{Population Loss for Ridge Regression for $\lambda >0$}

In this section, we assume $\alpha = \frac{n}{d} = \infty$ and we train our predictor by Ridge Regression with $\lambda >0$. In this case, by the same proof we did in \ref{appendix_preliminaries}:
\begin{equation}
    \hat{A}_\mathrm{App} =  I_{d} - Q_\mathrm{App}(\lambda) [\mathrm{diag}( Q_\mathrm{App}(\lambda))]^{-1}. 
\end{equation}
Then, the Approximation error $L(\mathrm{SSR})$ is given by: 
\begin{align}
    L(\mathrm{SSR}) & = \dfrac{1}{d}\mathrm{Tr} \left ([\mathrm{diag}( Q_\mathrm{App}(\lambda))]^{-1} Q_{\mathrm{App}}(\lambda) \Sigma  Q_{\mathrm{App}}(\lambda) [\mathrm{diag}( Q_\mathrm{App}(\lambda))]^{-1} \right ) \\
    & = \dfrac{1}{d} \mathrm{Tr} \left ([\mathrm{diag}( Q_\mathrm{App}(\lambda))]^{-2} Q_{\mathrm{App}}(\lambda) \Sigma  Q_{\mathrm{App}}(\lambda) \right ). 
\end{align}
Noting that $Q_{\mathrm{App}}(\lambda)$ and $\Sigma$ commute, we conclude: 
\begin{equation}
    L(\mathrm{SSR})  = \dfrac{1}{d}\mathrm{Tr} \left ([\mathrm{diag}( Q_\mathrm{App}(\lambda))]^{-2} \Sigma  Q_{\mathrm{App}}(\lambda)^2 \right )
\end{equation}
By using the fact that $[(\mathrm{diag}( Q_\mathrm{App}(\lambda)))]^{-2}$ is diagonal, we can summarize this as: 
\begin{equation}
    L(\mathrm{SSR}) = \dfrac{1}{d}\sum_{\ell = 1}^{d} \dfrac{(\Sigma Q_\mathrm{App}(\lambda)^2)_{\ell, \ell}}{[Q_\mathrm{App}(\lambda)_{\ell, \ell}]^2}
    \label{eq:app_Ridge_1}
\end{equation}
We can also write this as: 
\begin{equation}
    L(\mathrm{SSR})  = \dfrac{1}{d}\sum_{\ell = 1}^{d} \dfrac{Q_\mathrm{App}(\lambda)_{\ell, \ell}- \lambda Q_\mathrm{App}(\lambda)_{\ell, \ell}^2}{[Q_\mathrm{App}(\lambda)_{\ell, \ell}]^2}. 
    \label{eq:app_Ridge_2}
\end{equation}
In the following, we will use both \Cref{eq:app_Ridge_1} and \Cref{eq:app_Ridge_2}. 

\subsection{Population Loss for Principal Component Analysis}

Taking $n \to \infty$, the PCA estimator for $p$ component becomes: 
\begin{equation}
    P_{p} = \dfrac{1}{d}\sum_{\ell = 1}^{p} u_{\ell}u_{\ell}^T, 
\end{equation}
where $u_{\ell} u_{\ell}^T$ are the $p$ eigenvectors of $\Sigma$ with bigger eigenvalues. Then, the Approximation error of PCA is: 

\begin{equation}
    L(\mathrm{PCA}) = \dfrac{1}{d}\mathrm{Tr} (P_{\geq p +1}  \Sigma),
    \label{eq:app_PCA}
\end{equation}
where $P_{\geq p +1}$ denotes the projection into the $d - p$ lowest eigenvectors of $\Sigma$. 

Using \Cref{eq:app_Ridge_1} and \Cref{eq:app_PCA}, we can now compare the Performance of PCA and Ridge in the population limit for different structures. 

\subsection{Comparison for a Spiked Covariance - Proof of \Cref{lemma:pca_comparison_spiked_cov}}

In this section we study the case where the covariance of the data is given by $\Sigma = I_{d} + \theta vv^T$, for $\theta \geq 0$, $\| v\|_2 = 1$. It will be useful to denote $\tau = 1 + \lambda$. This way: 
\begin{equation}
    Q_\mathrm{App}(\lambda) = (I_{d} + \theta vv^T + \lambda I_{d})^{-1} = \left ( \tau I_{d} +\theta vv^T \right )^{-1} = \dfrac{1}{\tau} I_{d} - \dfrac{\theta}{\tau (\tau + \theta)} vv^T,
\end{equation}
and for every $\ell \in [d]$, the diagonal is given by: 
\begin{equation}
    Q_\mathrm{App}(\lambda)_{\ell, \ell} = \dfrac{1}{\tau} - \dfrac{\theta}{\tau (\tau + \theta)} v_{\ell}^2 = \dfrac{1}{\tau} \left ( 1 - \dfrac{\theta}{\tau + \theta} v_{\ell}^2\right ). 
\end{equation}
In order to compute the numerator in \Cref{eq:app_Ridge_1}, we note that $\Sigma$ and $Q(\lambda)_\mathrm{App}^2$ commute. In particular, $v$ will be an eigenvector with eigenvalue $\dfrac{1+\theta}{(\tau + \theta)^2} $, and all other eigenvector will have eigenvalues $\frac{1}{\tau^2}$. Then: 
\begin{equation}
    \Sigma Q(\lambda)_\mathrm{App}^2 =  \dfrac{1}{\tau^2} I_{d} + \left (\dfrac{1+\theta}{(\tau + \theta)^2} - \frac{1}{\tau^2} \right )vv^T. 
\end{equation}
Then, we get: 
\begin{equation}
    (\Sigma Q(\lambda)_\mathrm{App}^2)_{\ell, \ell} = \dfrac{1}{\tau^2} + \left (\dfrac{1+\theta}{(\tau + \theta)^2} - \frac{1}{\tau^2} \right )v_{\ell}^2. 
\end{equation}
This way, for a spiked covariance $\Sigma = I_d + \theta vv^T$, \Cref{eq:app_Ridge_1} is given by: 
\begin{equation}
    L(\mathrm{SSR}) = \dfrac{1}{d}\sum_{\ell = 1}^{d}\dfrac{\dfrac{1}{\tau^2} + \left (\dfrac{1+\theta}{(\tau + \theta)^2} - \dfrac{1}{\tau^2} \right )v_{\ell}^2}{\dfrac{1}{\tau^2} \left ( 1 - \dfrac{\theta}{\tau + \theta} v_{\ell}^2\right )^2} = \dfrac{1}{d}\sum_{\ell = 1}^{d}\dfrac{1+ \tau^2\left (\dfrac{1+\theta}{(\tau + \theta)^2} - \dfrac{1}{\tau^2} \right )v_{\ell}^2}{\left ( 1 - \dfrac{\theta}{\tau + \theta} v_{\ell}^2\right )^2}
\end{equation}
At last, to simplify, note that: 
\begin{equation}
    \tau^2\left (\dfrac{1+\theta}{(\tau + \theta)^2} - \dfrac{1}{\tau^2} \right ) = \dfrac{\tau^2 (1 + \theta) -\tau^2 -2\tau \theta - \theta^2}{(\tau + \theta)^2} = \dfrac{\theta(\tau^2 -2\tau - \theta )}{(\tau + \theta)^2},
\end{equation}
so 
\begin{equation}
    L(\mathrm{SSR}) = \sum_{\ell = 1}^{d}\dfrac{1+ \dfrac{\theta(\tau^2 -2\tau - \theta )}{(\tau + \theta)^2}v_{\ell}^2}{\left ( 1 - \dfrac{\theta}{\tau + \theta} v_{\ell}^2\right )^2} = \sum_{\ell = 1}^{d}\dfrac{1+ av_{\ell}^2}{\left ( 1 - b v_{\ell}^2\right )^2},
\end{equation}
where we defined $a = \dfrac{\theta(\tau^2 -2\tau - \theta )}{(\tau + \theta)^2}$ and $b = \dfrac{\theta}{\tau + \theta}$. \\

On the other hand, for PCA with any number of directions $p \in [d]$, 
\begin{equation}
    L(PCA) = \frac{1}{d} (d - p) = ( 1- \frac{p}{d} ) 
\end{equation}
Let $\gamma: = \frac{p}{d} $. Then, we have: 
\begin{align}
    L(\mathrm{SSR}) - L(\mathrm{PCA}) &= \left ( \dfrac{1}{d}\sum_{\ell = 1}^{d}\dfrac{1+ a v_{\ell}^2}{\left ( 1 - b v_{\ell}^2\right )^2} \right ) -  (1- \frac{p}{d} )\\
    & = \dfrac{p}{d} + \dfrac{1}{d}\sum_{\ell = 1}^{d} \left ( \dfrac{1+ av_{\ell}^2}{\left ( 1 - b v_{\ell}^2\right )^2} - 1\right ). 
\end{align}
For each $\ell \in [d]$, denote
\begin{equation}
    T_{\ell} =  \dfrac{1+ av_{\ell}^2}{\left ( 1 - b v_{\ell}^2\right )^2} - 1. 
\end{equation}
Then: 
\begin{align}
    T_{\ell} & = \dfrac{1+ a v_{\ell}^2}{\left ( 1 - b v_{\ell}^2\right )^2} - 1 \\
    & = \dfrac{1+ a v_{\ell}^2 - (1 - bv_\ell^2)^2}{( 1 - b v_{\ell}^2 )^2} \\
    & = \dfrac{1 + a v_\ell^2 - 1 + 2bv_\ell^2 -b^2 v_\ell^4}{ (1 - b v_{\ell}^2 )^2} \\
    & = \dfrac{v_\ell^2 ( a +2b -b^2 v_\ell^2)}{ (1 - b v_{\ell}^2 )^2}.
    \label{eq:intermediate_step_T}
\end{align}
Now, by the definition of $a$ and $b$: 
\begin{align}
    a +2b -b^2 v_\ell^2 & =  \dfrac{\theta(\tau^2 -2\tau - \theta )}{(\tau + \theta)^2} + 2 \dfrac{\theta}{\tau + \theta} - \dfrac{\theta^2}{(\tau + \theta)^2} v_{\ell}^2 \\
    & = \dfrac{\theta}{(\tau + \theta)^2} \left ((\tau^2 + \theta (1 - v_{\ell}^2 )\right ). 
\end{align}
Since $\| v\|_2 = 1$, we have that $1 - v_{\ell}^2  \geq 0 $, and hence: 
\begin{equation}
    a +2b -b^2 v_\ell^2 \geq 0. 
\end{equation}
Then, also: 
\begin{equation}
    T_{\ell} =  \dfrac{v_\ell^2 ( a +2b -b^2 v_\ell^2)}{ (1 - b v_{\ell}^2 )^2}\geq 0. 
\end{equation}
In particular, since $\| v\|_2 = 1$, there exists an $\ell \in [d]$ such that $a +2b -b^2 v_\ell^2 > 0 $. Therefore: 
\begin{equation}
    L(\mathrm{SSR}) - L(\mathrm{PCA}) \geq \dfrac{p}{d} >0. 
\end{equation}
and in particular:
\begin{equation}
    L(\mathrm{SSR}) - L(\mathrm{PCA}) > 0, 
\end{equation}
and therefore PCA achieves a better generalization error than Self-Supervised Ridge Regression in the population limit. 

\subsection{Comparison for a Toeplitz Matrix - Proof of \Cref{prop:pca_comparison_toeplitz}}
\label{appendix_pca_comparison_toeplitz}

In this section we will sketch the proof of \Cref{prop:pca_comparison_toeplitz}. 
Let $\rho \in (0,1)$, and define the Toeplitz matrix $\Sigma \in \RR^{d \times d}$  by 
\begin{equation}
    \Sigma_{i,j} = \rho^{|i-j|}. 
\end{equation}
We want to compute the approximation error for this model which we know by  \Cref{eq:app_Ridge_2} is given by: 
\begin{equation}
    L(\mathrm{SSR}) = \dfrac{1}{d}\sum_{\ell = 1}^{d} \dfrac{Q_\mathrm{App}(\lambda)_{\ell, \ell} - \lambda Q_\mathrm{App}(\lambda)^2_{\ell, \ell}}{[Q_\mathrm{App}(\lambda)_{\ell, \ell}]^2}. 
\end{equation}
Note that: 
\begin{equation}
    Q_\mathrm{App}(\lambda)^2_{\ell ,\ell} = -\left ( \dfrac{d}{d\lambda} Q_\mathrm{App}(\lambda ) \right )_{\ell, \ell } = - Q_\mathrm{App}'(\lambda)_{\ell, \ell}. 
\end{equation}
Then: 
\begin{equation}
    L(\mathrm{SSR}) = \dfrac{1}{d}\sum_{\ell = 1}^{d} \dfrac{Q_\mathrm{App}(\lambda)_{\ell, \ell} + \lambda Q'_\mathrm{App}(\lambda)_{\ell, \ell}}{ [Q_\mathrm{App}(\lambda)_{\ell, \ell}]^2}.
\end{equation}
So, what we need to compute $L(\mathrm{SSR})$ is to compute $Q(\lambda)_{\ell,\ell}$ for each $\ell \in [d]$. At this point, we still haven't used the fact that $\Sigma$ is a Toeplitz matrix. In particular, in high dimensions, this matrix can be approximated by the following circulant matrix (\cite{PB20}, Appendix A.3) 
\begin{equation}
    \tilde{\Sigma}_{i,j} = \rho^{\min \{|i-j|, |i-j + d|, |i-j - d| \}}. 
\end{equation}
Note that this only changes the element in the border of $\Sigma$. From now on, we will work with denote  $\Sigma = \tilde{\Sigma}$, knowing that asymptotically both matrices share the same spectrum (see \cite{PB20}, Appendix 3). \\

This modified version of $\Sigma$ can be diagonalized by the Fourier transform, and the eigenvector are given by: 
\begin{equation}
    [v_k]_{\ell} = \exp \frac{2\pi i k \ell}{d},
\end{equation}
where for each $k$, note that both the real and complex parts of $v_k$ are eigenvectors.  
As $d$ grows, we can approximate the eigenvalues of $\Sigma$ by: 
\begin{equation}
    E(x) = \dfrac{1 - \rho^2}{1 + \rho^2 - 2\rho \cos(\pi x)}, 0\leq x \leq 1. 
    \label{eq:expression_eigenvalues_toeplitz}
\end{equation}
More precisely, for $\ell \in [d]$, $\lambda_{\ell} \sim E(\frac{\ell}{d} )$. Now, recall we want to know the value of $Q(\lambda)_{\ell, \ell}$, for $\ell \in [d]$. However, since $\Sigma$ is now a circulant matrix, we have that $ Q(\lambda) = (\Sigma + \lambda I )^{-1}$ is a circulant matrix. Hence, all of it's diagonal entries are the same. This way, we can write: 
\begin{equation}
Q(\lambda)_{\ell, \ell}  = \dfrac{1}{d} \sum_{k=1}^{d} Q(\lambda)_{k,k} = \dfrac{\mathrm{Tr}(Q(\lambda)) }{d}. 
\end{equation}
Then, for large $d$: 
\begin{equation}
    Q(\lambda)_{\ell, \ell} \sim \int_0^1 \dfrac{1}{\lambda + E(x)} dx = m_{\Sigma}(-\lambda),
\end{equation}
where $m(\cdot)$ is the Stletjes transform. Using the fact that
\begin{equation}
    t_\Sigma(z) = -zm_{\Sigma}(z) - 1 \implies m_\Sigma(z) = -\dfrac{t_\Sigma(z)}{z} - \dfrac{1}{z}, 
\end{equation}
for $t_\Sigma(z)$ being the T-transform (\cite{PB20}, Chapter 11) , we can use (\cite{PB20}, Equation A.33) to get: 
\begin{align}
    Q(\lambda)_{\ell, \ell} &\sim \dfrac{1}{\lambda} -\dfrac{1}{\lambda} \dfrac{1}{\sqrt{(\lambda + E^{-})(\lambda + E^{+})}} \\
\end{align}
where $E_+ = \frac{1 + \rho}{1-\rho}$, $E_- = \frac{1 - \rho}{1+\rho}$. Then: 
\begin{align}
    Q(\lambda)^2_{\ell, \ell} & = - \dfrac{d}{d\lambda} Q(\lambda)_{\ell, \ell} \\
    & = - \dfrac{d}{d\lambda} \left ( \dfrac{1}{\lambda} -\dfrac{1}{\lambda} \dfrac{1}{\sqrt{(\lambda + E^{-})(\lambda + E^{+})}}  \right ) \\
    & = \dfrac{1}{\lambda^2} -\dfrac{\sqrt{(\lambda + E^{-})(\lambda + E^{+})}+\dfrac{\lambda(2\lambda +E^{+} + E^{-})}{2\sqrt{(\lambda + E^{-})(\lambda + E^{+})}}}{\lambda^2 (\lambda + E^{-})(\lambda + E^{+}) } \\
    & =  \dfrac{1}{\lambda^2} - \dfrac{1}{\lambda^2\sqrt{(\lambda + E^{-})(\lambda + E^{+})}} - \dfrac{(2\lambda + E^{+} + E^{-})}{2\lambda (\lambda + E^{-})^{\frac{3}{2}}(\lambda + E^{+})^{\frac{3}{2}}}. 
\end{align}
Then: 
\begin{align}
    Q(\lambda)_{\ell, \ell} - \lambda Q(\lambda)^2_{\ell, \ell} & = \dfrac{1}{\lambda} - \dfrac{1}{\lambda \sqrt{(\lambda + E^{-})(\lambda + E^{+})}} - \left ( \dfrac{1}{\lambda} - \dfrac{1}{\lambda\sqrt{(\lambda + E^{-})(\lambda + E^{+})}} -\dfrac{(2\lambda + E^{+} + E^{-})}{2(\lambda + E^{-})^{\frac{3}{2}}(\lambda + E^{+})^{\frac{3}{2}}}\right ) \\ 
    & = \dfrac{(2\lambda + E^{+} + E^{-})}{2(\lambda + E^{-})^{\frac{3}{2}}(\lambda + E^{+})^{\frac{3}{2}}}
\end{align}
and: 
\begin{align}
    \dfrac{Q(\lambda)_{\ell, \ell} - \lambda Q(\lambda)^2_{\ell, \ell}}{[Q(\lambda)_{\ell, \ell}]^2} & =  \dfrac{\dfrac{(2\lambda + E^{+} + E^{-})}{2(\lambda + E^{-})^{\frac{3}{2}}(\lambda + E^{+})^{\frac{3}{2}}}}{ \dfrac{1}{\lambda^2}\left ( 1 -\dfrac{1}{\sqrt{(\lambda + E^{-})(\lambda + E^{+})}}\right )^2} \\ 
    & = \dfrac{\dfrac{(2\lambda + E^{+} + E^{-})}{2(\lambda + E^{-})^{\frac{3}{2}}(\lambda + E^{+})^{\frac{3}{2}}}}{ \dfrac{1}{\lambda^2}\left ( \dfrac{\sqrt{(\lambda + E^{-})(\lambda + E^{+})} -1}{\sqrt{(\lambda + E^{-})(\lambda + E^{+})}}\right )^2} \\
    & = \dfrac{\lambda^2 (2\lambda + E^{+} + E^{-})}{2(\lambda + E^{-})^{\frac{1}{2}}(\lambda + E^{+})^{\frac{1}{2}}\left ( \sqrt{(\lambda + E^{-})(\lambda + E^{+})} -1\right )^2}.
\end{align}
With this, we conclude
\begin{equation}
    L(\mathrm{SSR}) \sim \dfrac{\lambda^2 (2\lambda + E^{+} + E^{-})}{2(\lambda + E^{-})^{\frac{1}{2}}(\lambda + E^{+})^{\frac{1}{2}}\left ( \sqrt{(\lambda + E^{-})(\lambda + E^{+})} -1\right )^2},
\end{equation}
where we recall $E^{+} = \frac{1 + \rho}{1-\rho}$ and $E^{-} = \frac{1-\rho}{1 + \rho}$. Let 
\begin{equation}
    f(\lambda, \rho ) : =  \dfrac{\lambda^2 (2\lambda + E^{+} + E^{-})}{2(\lambda + E^{-})^{\frac{1}{2}}(\lambda + E^{+})^{\frac{1}{2}}\left ( \sqrt{(\lambda + E^{-})(\lambda + E^{+})} -1\right )^2}
    \label{eq:f_lambda_rho_toeplitz}
\end{equation}
so that
\begin{equation}
    L(\mathrm{SSR}) \sim f(\lambda, \rho ). 
    \label{eq:ridge_app_toeplitz}
\end{equation}
As a sanity check, note that if we take $\rho \to 0$, we get 
$f(\lambda, \rho ) \to 1$, which we know is the case for an isotropic covariance.

On the other hand, we can compute the approximation error for PCA for a number of direction $p$ as follows. Assume that $\frac{p}{d} : = \gamma$, as $d $ grows to infinity. Then, asymptotically, the approximation error of PCA will be: 
\begin{equation}
    L(\mathrm{PCA}) = 1 - \dfrac{1}{d}\sum_{j = 1}^{\lfloor \gamma d\rfloor} \lambda_j \sim  \int_\gamma^1 E(x)  dx.
\end{equation}
By computing this integral, we get: 
\begin{align}
    \int_\gamma^1 E(x)  dx
    & = \int_\gamma^1 \dfrac{1 - \rho^2}{1 + \rho^2 - 2\rho \cos(\pi x)} dx \\
    & = (1 - \rho^2) \int_\gamma^1 \dfrac{1}{1 + \rho^2 - 2\rho \cos(\pi x)}dx \\
    & = \dfrac{(1 - \rho^2)}{(1-\rho^2)} \left ( 1 - \dfrac{2}{\pi} \arctan \left ( \dfrac{1 + \rho}{1-\rho} \tan \dfrac{\pi \gamma}{2}\right ) \right ) \\
    & =  1 - \dfrac{2}{\pi} \arctan \left ( E^{+} \tan \dfrac{\pi \gamma}{2}\right ),
\end{align}
so 
\begin{equation}
    L(\mathrm{PCA}) \sim \left ( 1 - \dfrac{2}{\pi} \arctan \left ( E^{+} \tan \dfrac{\pi \gamma}{2}\right ) \right ). 
\end{equation}
With this, we conclude that, asymptotically in $d$,
\begin{equation}
    L(\mathrm{PCA}) < L (\mathrm{SSR}) \iff \left ( 1 - \dfrac{2}{\pi} \arctan \left ( \dfrac{1 + \rho}{1-\rho} \tan \dfrac{\pi \gamma}{2}\right ) \right ) \leq f(\lambda, \rho ).
\end{equation}
Now, we take the limit as $\lambda \to 0$ for $f(\lambda, \rho)$. Let $F(\lambda) = \sqrt{(\lambda + E^{-})(\lambda + E^{+})} = \sqrt{\lambda^2 + (E^{+} + E^{-})\lambda +1}$. Then: 
\begin{align}
    f(\lambda, \rho) & = \dfrac{\lambda^2 (2\lambda + E^{+} + E^{-})}{2(\lambda + E^{-})^{\frac{1}{2}}(\lambda + E^{+})^{\frac{1}{2}}\left ( \sqrt{(\lambda + E^{-})(\lambda + E^{+})} -1\right )^2} \\
    & = \dfrac{\lambda^2 (2\lambda + E^{+} + E^{-})}{2F(\lambda)\left ( F(\lambda) -1\right )^2}.
\end{align}
Now: 
\begin{equation}
    F(\lambda) -1 = \dfrac{F(\lambda)^2 - 1}{F(\lambda) + 1},
\end{equation}
so: 
\begin{align}
    f(\lambda, \rho) & = \dfrac{\lambda^2 (2\lambda + E^{+} + E^{-})(F(\lambda) + 1)^2}{2F(\lambda)\left ( F(\lambda)^2 -1\right )^2} = \dfrac{\lambda^2 (2\lambda + E^{+} + E^{-})(F(\lambda) + 1)^2}{2F(\lambda)\lambda^2\left ( \lambda + (E^{+} + E^{-})\right )^2} \\
    & =  \dfrac{(2\lambda + E^{+} + E^{-})(F(\lambda) + 1)^2}{2F(\lambda)\left ( \lambda + (E^{+} + E^{-})\right )^2}.
\end{align}
By noting that $F(\lambda) \to 1$ as $\lambda \to 0$, we get: 
\begin{equation}
    f(\lambda, \rho) \to \dfrac{2}{(E^{+} + E^{-} )}. 
\end{equation}
Then, for any fixed $\rho \in (0,1)$, as $\lambda \to 0$ we get: 
\begin{equation}
    L(\mathrm{PCA}) < L (\mathrm{SSR}) \iff \left ( 1 - \dfrac{2}{\pi} \arctan \left ( \dfrac{1 + \rho}{1-\rho} \tan \dfrac{\pi \gamma}{2}\right ) \right ) \leq \dfrac{2}{(E^{+} + E^{-} )}
\end{equation}

Recall that $E^{-} = \sfrac{1}{E^{+}}$, so: 
\begin{equation}
    L(\mathrm{PCA}) < L (\mathrm{SSR}) \iff \left ( 1 - \dfrac{2}{\pi} \arctan \left ( E^{+} \tan \dfrac{\pi \gamma}{2}\right ) \right ) \leq \dfrac{2 E^{+} }{(E^{+} )^2 + 1} 
\end{equation}
Now, the LHS of this inequality is decreasing in $\gamma$. This means that, for a fixed value of $E^{+}(\rho)$, there exists a minimum value $\gamma_{\mathrm{min}}(\rho)$ such that $L(\mathrm{PCA}) < L(\mathrm{SSR})$ for all $\gamma \geq \gamma_{\mathrm{min}}(\rho)$. To find this value, we have: 
\begin{align}
     & \left ( 1 - \dfrac{2}{\pi} \arctan \left ( E^{+} \tan \dfrac{\pi \gamma}{2}\right ) \right ) \leq \dfrac{2 E^{+} }{(E^{+} )^2 + 1} \\
     \iff & \frac{\pi}{2} \left (1- \frac{2E^+}{(E^+)^2 + 1} \right ) \leq \arctan \left ( E^{+} \tan \dfrac{\pi \gamma}{2}\right ) \\
     \iff & \dfrac{2}{\pi}\arctan \left ( \frac{1}{E^+}\tan \left (  \frac{\pi}{2} \left (1- \frac{2E^+}{(E^+)^2 + 1} \right )\right ) \right ) \leq \gamma \\
     \iff & \dfrac{2}{\pi}\arctan \left ( \frac{1}{E^+}\tan \left (  \frac{\pi}{2} \left (\frac{(E^+ - 1)^2}{(E^+)^2 + 1} \right )\right ) \right ) \leq \gamma \\
     \iff  &  \dfrac{2}{\pi}\arctan \left ( \frac{1-\rho}{1+\rho}\tan \left (  \frac{\pi}{2} \left (\dfrac{(2\rho)^2}{(1+\rho)^2 + (1-\rho)^2} \right )\right ) \right ) \leq \gamma 
     \label{eq:phase_transition_curve}
\end{align}

Note that $\gamma$ increases as $\rho$ moves closer to $1$. This means that, as the covariance matrix becomes more structured (or, in terms of the AR(1) model,  the dependence on the past is stronger), PCA needs a growing number of directions to achieve the same generalization error as self-supervised Ridge Regression.

\section{Generalization Error for Toeplitz Matrices for finite $\alpha = \frac{n}{d}$}
\label{appendix_toeplitz_finite_alpha}

Recall that, by \Cref{theorem:det_equivalent_gen_error}, we have the following asymptotic limit for $L(\hat{A})$: 
\begin{align}
    L(\hat{A}) & \to L_1 \left (1 +  \dfrac{\mathrm{df}^{\Sigma}_2(\kappa_{\star}(\lambda)) }{n - \mathrm{df}^{\Sigma}_2(\kappa_{\star}(\lambda))} \right ) 
    \label{eq:gen_error_toeplitz_1}
\end{align}
where 
\begin{align}
        L_1  = \frac{1}{d}\mathrm{Tr} \left [ \bar{D}^2 (\Sigma + \kappa_{\star}(\lambda) I_{d})^{-1} \Sigma (\Sigma + \kappa_{\star}(\lambda) I_{d})^{-1}\right ],\\
        \bar{D} = \bar{D}(\Sigma, \kappa_{\star}(\lambda) ) :=[\mathrm{diag}((\Sigma + \kappa_{\star}(\lambda) I_d)^{-1}))]^{-1},
        \label{eq:gen_error_toeplitz_2}
\end{align}
and $\kappa_{\star}(\lambda)$ satisfies the self-consistent equation: 
\begin{equation}
        \kappa = \lambda + \frac{\kappa}{n} \mathrm{Tr} \left\{\Sigma (  \Sigma + \kappa I_{d})^{-1}\right\}.
\end{equation}

We will first compute the degrees of freedom for any $\kappa >0$: 
\begin{equation}
    \mathrm{df}_2(\kappa) = \mathrm{Tr} \left ( \Sigma^2 (\Sigma + \kappa I_{d})^{-2}\right ) = \sum_{\ell=1}^{d} \dfrac{\lambda_{\ell}^2}{(\lambda_{\ell} + \lambda )^2},
\end{equation}
for $\lambda_{1}, \dots, \lambda_{d}$ eigenvalues of $\Sigma$, 
and for any $\kappa >0$. Recall that, as we saw in the last section, as $d$ grows we can approximate the eigenvalues of $\Sigma$ by: 
\begin{equation}
    E(x) = \dfrac{1 - \rho^2}{1 + \rho^2 - 2\rho \cos(\pi x)}, 0\leq x \leq 1. 
    \label{eq:expression_eigenvalues_toeplitz}
\end{equation}
More precisely, for $\ell \in [d]$, $\lambda_{\ell} \sim E(\frac{\ell}{d} )$. Then, we have: 
\begin{align}
    \frac{1}{d}\mathrm{df}_{2}^{\Sigma}(\kappa) & = \sum_{\ell = 1}^{d}  \dfrac{\lambda_{\ell}^2}{(\lambda_{\ell} + \lambda )^2} \\
    & \sim \int_{0}^{1} \dfrac{E(x)^2}{(E(x) + \kappa)^2} dx. 
\end{align}
Developing this integral we get:
\begin{align}
     \int_{0}^{1} \dfrac{E(x)^2}{(E(x) + \kappa)^2} dx & = \int_0^1 \dfrac{\left ( \dfrac{1 - \rho^2}{1 + \rho^2 - 2\rho \cos(\pi x)}\right )^2}{\left (\dfrac{1 - \rho^2}{1 + \rho^2 - 2\rho \cos(\pi x)} + \kappa  \right )^2}dx \\
     & = \int_0^1 \dfrac{\left ( 1 - \rho^2\right )^2}{\left (1 - \rho^2 + \kappa(1 + \rho^2 - 2\rho \cos(\pi x))  \right )^2} dx \\ 
     & = (1 - \rho^2)^{2} \int_0^1 \dfrac{1}{\left (1-\rho^{2}+ \kappa(1 + \rho^2 ) - 2\kappa \rho \cos(\pi x))  \right )^2} dx.
     \label{eq:intermediate_step_finite_rho}
\end{align}
To compute this last integral, we recall that, by Equation A.31 in \cite{PB20}: 
\begin{equation}
    \int_{0}^{1} \dfrac{dx}{c - d\cos(\pi x) } = \dfrac{1}{\sqrt{c-d} \sqrt{c+d}}. 
\end{equation}
Then: 
\begin{align}
    \int_{0}^{1} \dfrac{dx}{(c - d\cos(\pi x) )^{2}} & = - \int_0^1 \partial_{c} \dfrac{dx}{c - d \cos (\pi x)} \\
    & = - \partial_c\dfrac{1}{\sqrt{c-d} \sqrt{c+d}} \\
    & = \dfrac{\sqrt{\dfrac{c+d}{c-d}} + \sqrt{\dfrac{c-d}{c+d}} }{(c-d)(c+d)} \\
    & = \dfrac{c}{\left ( (c-d)(c+d) \right )^{\frac{3}{2}}}.
\end{align}
Then, if we go back to \Cref{eq:intermediate_step_finite_rho}:
\begin{align}
    \int_{0}^{1} \dfrac{E(x)^2}{(E(x) + \kappa)^2} dx & \sim (1 - \rho^2)^{2} \int_0^1 \dfrac{1}{\left (1-\rho^{2}+ \kappa(1 + \rho^2 ) - 2\kappa \rho \cos(\pi x))  \right )^2} dx \\
    & = (1 - \rho^2)^{2} \dfrac{1-\rho^{2}+ \kappa(1 + \rho^2 )}{ \left ((1-\rho^{2}+ \kappa(1 + \rho^2 ) +2\kappa \rho)(1-\rho^{2}+ \kappa(1 + \rho^2 ) - 2\kappa \rho)\right )^{\frac{3}{2}} } \\
    & = (1 - \rho^2)^2 \dfrac{1 - \rho^2 + \kappa(1 + \rho^2)}{\left ( (1 - \rho^2 + \kappa(1 + \rho)^2) (1 - \rho^2 + \kappa(1 - \rho)^2)\right )^{\frac{3}{2}}}.
\end{align}
Then: 
\begin{equation}
    \mathrm{df}_{2}^{\Sigma}(\kappa) \sim d (1 - \rho^2)^2 \dfrac{1 - \rho^2 + \kappa(1 + \rho^2)}{\left ( (1 - \rho^2 + \kappa(1 + \rho)^2) (1 - \rho^2 + \kappa(1 - \rho)^2)\right )^{\frac{3}{2}}}. 
\end{equation}
Note that, taking $\rho \to 0$, we get $\mathrm{df}_2(\kappa) \to \dfrac{d}{(1 + \kappa)^{2}}$, which corresponds to the isotropic case. Taking $\rho \to 1$, for the numerator the leading order is 
\begin{equation}
     (1 - \rho^2)^2 (1 - \rho^2 + \kappa(1 + \rho^2)) \sim  2\kappa  (1 - \rho^2)^{2},
     \label{eq:leading_order_num}
\end{equation}
while for the denominator is: 
\begin{equation}
\left ( (1 - \rho^2 + \kappa(1 + \rho)^2) (1 - \rho^2 + \kappa(1 - \rho)^2)\right )^{\frac{3}{2}} \sim (4\kappa (1-\rho^{2})  )^{\frac{3}{2}}.
\label{eq:leading_order_den}
\end{equation}
Then, when $\rho \to 1$, by \Cref{eq:leading_order_num} and \Cref{eq:leading_order_den}: 
\begin{equation}
    \mathrm{df}_{2}^{\Sigma}(\kappa) \sim d \dfrac{2\kappa  (1 - \rho^2)^{2}}{(4\kappa (1-\rho^{2})  )^{\frac{3}{2}}} \sim d\dfrac{\sqrt{1- \rho^{2}}}{4\sqrt{\kappa }}.
    \label{eq:dof_2_toeplitz_app}
\end{equation}
Now, as $\lambda \to 0$, \Cref{eq:gen_error_toeplitz_1} and \Cref{eq:gen_error_toeplitz_2} become:
\begin{align}
    L(\hat{A}) & \to L_1 \left (1 +  \dfrac{\mathrm{df}^{\Sigma}_2(\kappa_{\star}(0)) }{n - \mathrm{df}^{\Sigma}_2(\kappa_{\star}(0))} \right ) 
    \label{eq:gen_error_toeplitz_1_1}
\end{align}
where 
\begin{align}
        L_1  = \frac{1}{d}\mathrm{Tr} \left [ \bar{D}^2 (\Sigma + \kappa_{\star}(0) I_{d})^{-1} \Sigma (\Sigma + \kappa_{\star}(0) I_{d})^{-1}\right ],\\
        \bar{D} = \bar{D}(\Sigma, \kappa_{\star}(0) ) :=[\mathrm{diag}((\Sigma + \kappa_{\star}(0) I_d)^{-1}))]^{-1}.
        \label{eq:gen_error_toeplitz_2_2}
\end{align}
For $n < d$, we have $\alpha < 1$ and $\kappa(0) > 0$, but is very close to $0$ when $\alpha  \to 1$. Therefore in this case, $L_1$ is almost constant when $\alpha$ is close to $1$. Then, as $\alpha \to 1$
\begin{equation}
    L(\hat{A}) \sim C ( 1 + \dfrac{1}{\frac{n}{\mathrm{df}_2^{\Sigma}(\kappa(0))}-1}),
\end{equation}
and from \Cref{eq:dof_2_toeplitz_app}:
\begin{equation}
      L(\hat{A}) \sim C ( 1 + \dfrac{1}{\dfrac{4 \alpha \sqrt{\kappa(0)}}{\sqrt{1- \rho^{2}}}-1}). 
\end{equation}
Then, as $\rho \to 1$, the rate at which the generalization error goes to infinity when $\lambda \to 0$ changes, becoming slower.

\section{BBP Transition - Proof of \Cref{prop:BBP_transition}}
\label{appendix_bbp}

In what follows, we study the case where the population covariance takes the form of
\begin{equation}\label{eq:Sigma_spike}
    \Sigma = I_{d} + \theta v v^\top ,
\end{equation}
where $\Sigma_0$ is the ``bulk'' (diagonal) matrix and $v \in \RR^d$ is a unit-norm vector. 

By the same argument we did in Appendix \ref{section_diagonal_ridgeless_limit} for the case where there is no spake, we have the following deterministic equivalent for the resolvent: 
\begin{equation}
    \bar{Q}(\lambda) = (\dfrac{1}{\nu + 1} I_{d} + \lambda I_{d})^{-1},
\end{equation}
and $\nu$ satisfies: 
\begin{equation}
    \nu = \dfrac{1}{n} \mathrm{Tr} \left (\left (\dfrac{1}{\nu +1}  + \lambda I_{d} \right )^{-1} \right ).
    \label{eq:self_consistent_nu}
\end{equation}
Taking $\lambda \to 0$, when $\alpha >1$ we get that $\nu_{\mathrm{isotropic}} = \frac{1}{\alpha - 1}$, and the diagonal matrix $\bar{D}$ becomes: 
\begin{equation}
   \bar{D}_{\mathrm{isotropic}} = \dfrac{1}{\nu + 1} I_{d}.
\end{equation}
Now, if we add a rank-one update $\theta vv^T$ to a diagonal matrix by \Cref{eq:self_consistent_nu} we have that 
\begin{equation}
    \mathrm{\nu}_{\mathrm{spike}} = \mathrm{\nu}_{\mathrm{isotropic}} = \frac{1}{\alpha - 1}. 
\end{equation}
Then, as $\lambda \to 0$, the deterministic equivalent for the resolvent is: 
\begin{align}
     \bar{Q}(\lambda) & = (\dfrac{1}{\nu + 1} \Sigma + \lambda I_{d})^{-1} \\
     & = \dfrac{\alpha}{\alpha-1} (I_{d} + \theta vv^T)^{-1} \\
     & = \dfrac{\alpha}{\alpha-1}(I_{d} - \dfrac{\theta}{1+ \theta} vv^T ).
\end{align}
Then, in this case, for each $\ell \in [d]$:
\begin{align}
    (\bar{D}_{\mathrm{spike}} )_{\ell, \ell} & = \left [ \mathrm{diag} (\bar{Q}(\lambda))\right ]^{-1}_{\ell, \ell} \\
    & = \dfrac{\alpha-1}{\alpha} \left [ \mathrm{diag}(I_{d} - \dfrac{\theta}{1+ \theta} vv^T )^{-1} \right ]^{-1}_{\ell, \ell} \\
    & = \dfrac{\alpha-1}{\alpha} \left ( 1 - \dfrac{\theta}{1+ \theta} v_{\ell} v_{\ell}\right )^{-1}. 
    \label{eq:spiked_det_matrix}
\end{align}

\begin{remark}
    Note that if the vector $v$ is delocalized (with respect to $I_{d}$), then the second term is asymptotically negligible. On the other hand, if it's localized (such as $v = e_{k}$, for some $k \in [d]$, then this term is not negligible. 
\end{remark}

We can now examine the potential spike in our problem. Recall the matrix $\hat{A}$ can be written as: 
\begin{equation}
    \hat A = I - Q [\diag(Q)]^{-1}
\end{equation}
As we saw before, we have that
\begin{equation}
    \| \hat{A} - (I_{d} - Q(\lambda) \bar{D}) \|_\mathrm{op} \to 0,
\end{equation}
as $d$ grows. Therefore, for the purpose of this section we can study $I_{d} - Q(\lambda) \bar{D}$ instead of $\hat{A}$. The eigenvalues of this matrix are in a one-to-one correspondence with those of 
\begin{equation}
  \bar{D}^{\frac{1}{2}}Q\bar{D}^{\frac{1}{2}}. 
\end{equation}
Specifically, if $v$ is an eigenvector of $\bar{D}^{\frac{1}{2}}Q\bar{D}^{\frac{1}{2}}$ with eigenvalue $s$, then $\bar{D}^{-\frac{1}{2}}v$ is an eigenvector of $I_{d} - Q(\lambda) \bar{D}$ with eigenvalue $1-s$. In what follows, we study $\bar{D}^{\frac{1}{2}}Q\bar{D}^{\frac{1}{2}}$. 

We will focus on the case where the spike $v \in \RR^{d}$ is generic, i.e $\| v \|_{\infty} = o_{d}(1) $. Then, in this case, the diagonal matrix $\bar{D}$ is the same as the diagonal case:
\begin{equation}
     (\bar{D}_{\mathrm{spike}}^{\frac{1}{2}} )_{\ell, \ell} =  \sqrt{\dfrac{\alpha-1}{\alpha}} I_{d},
     \label{eq:D_spike}
\end{equation}
and the spike does not modify it.

In the isotropic case, as stated in \Cref{corollary:universality}, the bulk of the spectrum of $\bar{D}^{\frac{1}{2}} Q(\lambda) \bar{D}^{\frac{1}{2}}$ is bounded in the interval: 
\begin{equation}
    I :=\left[\frac{\sqrt \alpha-1}{\sqrt \alpha + 1}, \frac{\sqrt \alpha + 1}{\sqrt \alpha - 1}\right].
\end{equation}
The spike will appear in the left of $I$. By \cref{thm:det_equivalent_A_hat}, a deterministic equivalent for the resolvent 
\begin{equation}
    G(z) := (\bar{D}^{\frac{1}{2}} Q(\lambda) \bar{D}^{\frac{1}{2}} - zI_{d})^{-1}, z \in \C^{+}, 
    \label{eq:resolvent_D_Q_D_spike}
\end{equation}
is given by: 
\begin{equation}
    \mathcal{G}(z) = \left(\bar{D}^{\frac{1}{2}}\left(\frac{1}{1-\chi} \Sigma + \lambda I_d\right)^{-1} \bar{D}^{\frac{1}{2}} - z I_d\right)^{-1},
    \label{eq:G_equiv_spike}
\end{equation}
where $\chi$ solves: 
\begin{align}
     \chi & = \frac{1}{n} \mathrm{Tr}\left(\Sigma( - \frac{1}{1-\chi} \Sigma + \bar{D}/z)^{-1}\right) 
     \label{eq:self_consistent_spike}
\end{align}
Replacing \cref{eq:D_spike} in \cref{eq:G_equiv_spike} we get: 
\begin{align}
    \mathcal{G}(z)
    & = \left (\dfrac{\alpha-1}{\alpha } \left ( \frac{1}{1-\chi} ( I_{d}  + \theta vv^T) \right )^{-1} - zI_{d} \right)^{-1} \\
    & = \dfrac{\alpha}{(\alpha-1)(1-\chi)} \left ( (I_{d} + \theta vv^T )^{-1} - \dfrac{z\alpha}{(\alpha - 1)(1-\chi)} I_{d}\right )^{-1} \\
    & = \dfrac{\alpha}{(\alpha-1)(1-\chi)} \left ( \left ( 1 - \dfrac{z\alpha}{(\alpha - 1)(1-\chi)}\right )I_{d} - \dfrac{\theta}{\theta + 1} vv^T \right )^{-1} 
\end{align}
The spike appears at $z$ such that this matrix is singular. This happens when: 
\begin{align}
    & \left ( 1 - \dfrac{z\alpha}{(\alpha - 1)(1-\chi)}\right ) - \dfrac{\theta}{\theta + 1} = 0  \\
    \iff & \dfrac{z\alpha}{(\alpha - 1)(1-\chi)} = 1 - \dfrac{\theta}{\theta + 1} \\
    \iff & \dfrac{z\alpha}{(\alpha - 1)(1-\chi)} = \dfrac{1}{\theta + 1}. 
    \label{eq:top_eigenvalue_equation}
\end{align}
Note that for $\theta = 0$, the left-limit of the bulk is $\dfrac{\sqrt{\alpha} - 1}{\sqrt{\alpha} +1}$, and in this case, $\chi_{0}$ solves the quadratic equation: 
\begin{equation}
    \chi^{2}_0 +(z-1)\chi + \dfrac{z}{\alpha - 1} =0. 
    \label{eq:self_consistent_isotropic}
\end{equation}
Note that, as $n$ and $d$ grow to infinity, the solution of \Cref{eq:self_consistent_spike}  and \Cref{eq:self_consistent_isotropic} are asymptotically the same, as the spike will not change the normalized trace in the limit. Then, at $z = \dfrac{\sqrt{\alpha} - 1}{\sqrt{\alpha} +1}$ we have: 
\begin{align}
    &\chi_{0}^2 + \left ( \dfrac{\sqrt{\alpha} - 1}{\sqrt{\alpha} +1} -1 \right ) \chi_0 +  \dfrac{ \dfrac{\sqrt{\alpha} - 1}{\sqrt{\alpha} +1}}{\alpha-1} = 0 \\
    \iff & \chi_{0}^2 - \dfrac{2}{\sqrt{\alpha} +1}  \chi_0 +  \dfrac{1}{(\sqrt{\alpha}+1)^2} = 0 \\
    & \implies \chi_0 = \dfrac{1}{\sqrt{\alpha} + 1}. 
\end{align}
Then, \Cref{eq:top_eigenvalue_equation}, for the critical $\theta_{c}$ becomes: 
\begin{align}
    & \dfrac{\dfrac{\sqrt{\alpha} - 1}{\sqrt{\alpha} +1}\alpha}{(\alpha - 1)(1-\dfrac{1}{\sqrt{\alpha} + 1})} = \dfrac{1}{\theta_{c} + 1} \\
    & \implies \dfrac{\sqrt{\alpha}}{(\sqrt{\alpha} + 1)} = \dfrac{1}{\theta_{c} + 1} \\
    & \implies 1 + \theta_c = 1 +\dfrac{1}{\sqrt{\alpha}},
\end{align}
so we conclude $\theta_c  =\dfrac{1}{\sqrt{\alpha}} $. Then, for all $\alpha \geq \frac{1}{\sqrt{\alpha}}$, there will be an outlier in the spectrum.

\end{document}